\documentclass[letterpaper]{article}
\usepackage[preprint]{aaai2027}
\usepackage[hyphens]{url}
\usepackage{graphicx}
\usepackage{natbib}
\usepackage{caption}
\usepackage{booktabs}
\usepackage{array}
\usepackage{amsmath}
\usepackage{amssymb}
\usepackage{amsfonts}
\usepackage{amsthm}
\usepackage{mathtools}
\usepackage{nicefrac}
\usepackage{xcolor}
\definecolor{examplebg}{RGB}{237,242,247}
\usepackage{tikz}
\usetikzlibrary{arrows.meta,positioning,calc}
\usepackage{pgfplots}
\pgfplotsset{compat=1.17}
\usepackage{placeins}

\title{CAGE: Certified Authorization under Typed-Return Uncertainty\\ for Tool-Using Agents}
\author{
    Blaise Delattre\textsuperscript{\rm 1},\quad
    Cong Wang\textsuperscript{\rm 2},\quad
    Yang Cao\textsuperscript{\rm 1}
}
\affiliations{
    \textsuperscript{\rm 1}Department of Computer Science, School of Computing, Institute of Science Tokyo, Tokyo, Japan\\
    \textsuperscript{\rm 2}City University of Hong Kong\\
    \texttt{delattre.b.aa@m.titech.ac.jp}
}

\newtheorem{definition}{Definition}
\newtheorem{theorem}{Theorem}
\newtheorem{proposition}{Proposition}
\newtheorem{remark}{Remark}

\newcommand{\Safe}{\operatorname{Safe}}
\newcommand{\Allow}{\operatorname{Allow}}
\newcommand{\Disc}{D_{\mathrm{disc}}}
\newcommand{\Ball}{B_{d,\varepsilon}}
\newcommand{\Ndisc}{\mathcal N_d}
\newcommand{\Rallow}{R_{\mathrm{allow}}}
\newcommand{\method}{CAGE}
\newcommand{\Callow}{C_{\mathrm{allow}}}
\newcommand{\Uallow}{U_{\mathrm{allow}}}

\definecolor{AGInk}{HTML}{183238}
\definecolor{AGMuted}{HTML}{52666B}
\definecolor{AGPanel}{HTML}{F1F5F2}
\definecolor{AGLine}{HTML}{D8E1DD}
\definecolor{AGTeal}{HTML}{008C89}
\definecolor{AGTealSoft}{HTML}{E7F4F2}
\definecolor{AGAmber}{HTML}{D99A1B}
\definecolor{AGAmberSoft}{HTML}{FFF2D3}
\definecolor{AGSignal}{HTML}{C83F36}
\definecolor{AGSignalSoft}{HTML}{F9E3DF}
\definecolor{AGBlueSoft}{HTML}{E9EEF8}
\definecolor{AGGreenSoft}{HTML}{E8F4E8}

\tikzset{
  agbox/.style={
    draw=AGLine,
    line width=0.45pt,
    rounded corners=2pt,
    align=center,
    inner xsep=5pt,
    inner ysep=4pt,
    minimum height=10mm,
    fill=AGPanel,
    text=AGInk
  },
  agsmall/.style={font=\scriptsize, text=AGMuted},
  agtitle/.style={font=\scriptsize\bfseries, text=AGInk},
  agarrow/.style={
    -{Stealth[length=2.4mm,width=1.7mm]},
    line width=0.65pt,
    draw=AGInk
  },
  agdasharrow/.style={
    -{Stealth[length=2.4mm,width=1.7mm]},
    line width=0.65pt,
    dashed,
    draw=AGSignal
  }
}

\begin{document}

\maketitle

\begin{abstract}
    Tool-using LLM agents act on typed tool returns, records pairing provenance and categorical fields with numerical values.
    Runtime permission gates generally authorize the observed return and action, leaving the decision unprotected against small errors in how the return was bound to its source.
    We ask whether a candidate action stays authorized over a declared neighborhood of plausible correctly bound returns: one admissible binding fault plus bounded numerical drift.
    We prove that certifying the categorical and numerical channels separately does not compose: perturbations that are safe on each channel alone can jointly turn the same action unsafe.
    \method{} certifies this joint neighborhood directly, enumerating the discrete branches exactly and certifying the continuous perturbation within each branch.
    Across synthetic, policy-as-code, regulatory, and real-transaction settings, \method{} removes the in-budget false allows that accurate pointwise gates admit, while keeping a useful fraction of decisions autonomous.
    When the policy is executable, \method{}-Exact certifies the policy itself; otherwise \method{}-Lip and \method{}-RS certify a learned gate under an explicit, measured fidelity assumption.
\end{abstract}

\section{Introduction}

LLM agents operate in a loop (reason, emit a structured tool request, observe, continue~\cite{yao2023react}), so tool observations become inputs to future actions: external bytes influence control flow~\cite{greshake2023indirect,ruan2024toolemu}.
Deployed agent harnesses increasingly place a runtime permission gate between the model and tool execution~\cite{liu2026claudecode,ji2026permissiongate}; current gates are heuristic, pre-execution, and defined over the proposed call: they decide whether it may run, and leave open whether the agent may safely act on what it returns.
Existing defenses concentrate on the untrusted text channel: sanitizing or spotlighting tool outputs~\cite{das2025commandsans,hines2024spotlighting}, structured prompting and instruction hierarchies~\cite{chen2025struq,wallace2024instructionhierarchy}, and information-flow control by design~\cite{debenedetti2025camel}; these act on the text or data-flow channel, complementary to the typed authorization decision we certify (\S\ref{sec:e2e}).
\begin{figure}[t]
    \centering
    \includegraphics[width=\linewidth]{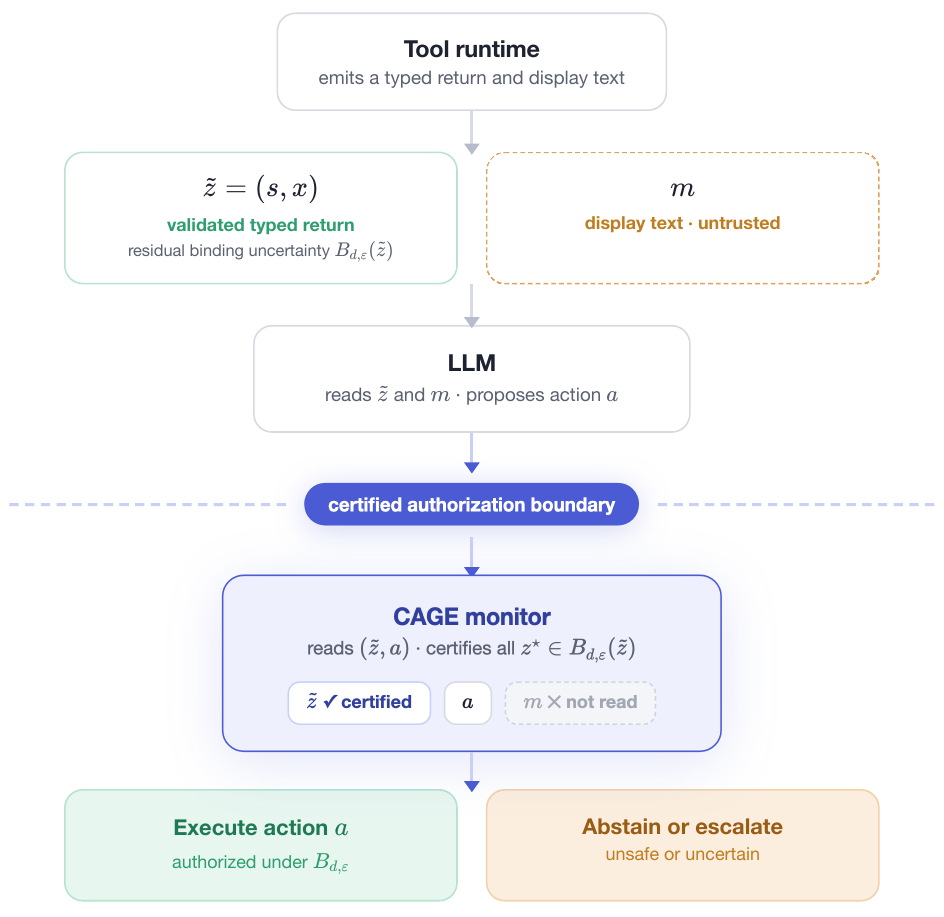}
    \caption{Post-tool-return authorization boundary. The LLM reads the validated record \(\tilde z=(s,x)\) and untrusted text \(m\), which can steer the proposal \(a\); the gate reads only \((\tilde z,a)\): the text channel is excluded architecturally; the typed channel is validated, its residual binding uncertainty certified.}
    \label{fig:post-return-boundary}
\end{figure}
\smallskip
\noindent\fcolorbox{gray!45}{examplebg}{\parbox{0.94\columnwidth}{\footnotesize
        \textbf{Example (joint-gap witness).} The served record reads \(\texttt{source}=\texttt{credit\_check}\), \(\texttt{risk\_score}=0.46\), and the agent proposes \(\texttt{approve\_transaction}\); under credit-check provenance the approval threshold is loose.
        The tag is stale: the record was produced by the sanctions screen, whose threshold is strict; the fresh score sits slightly higher, within validation tolerance.
        The observed record is safe at the point and passes each single-channel check; a state it plausibly stands for is not.
        The pointwise predicate authorizes the served record; evaluation over \(\Ball\) identifies the unsafe plausible return.}}
\smallskip

Existing runtime hooks can host the required decision, but permission gates, runtime enforcers, and LLM judges all evaluate the realized return--action pair pointwise.
We change the object of authorization from the observed point \((\tilde z,a)\) to the joint neighborhood \(\Ball(\tilde z)\): a candidate action is authorized only when it is safe for every correctly bound return in \(\Ball(\tilde z)\).

We therefore study the node \((z,a)\mapsto \Allow(z,a)\), where \(z=(s,x)\) is a typed tool return and \(a\) is a candidate downstream action; the LLM may still propose \(a\), and the runtime gate decides whether executing it is authorized under a specified perturbation model, abstention routing the residual to human review.

The central vulnerability is a joint-gap attack: a return-action pair may be clean-safe at the observed point and pass discrete-only and continuous-only checks, yet fail under a joint perturbation, because the discrete part of the return changes the semantics of the numerical fields, as in the screening example above.
We call such points \emph{joint-gap witnesses} (category \(C\) of the oracle taxonomy): safe at the observed point and under each marginal check, yet unsafe for some \((s',x')\in\Ball(s,x)\).
The defense, our monitor \textbf{\method{}} (Certified Authorization Gate for Execution), certifies the actual joint perturbation set: enumerate the finite discrete neighborhood exactly and certify each continuous branch with a sound backend.
When the policy is executable, \method{}-Exact certifies it directly; otherwise, \method{}-Lip or \method{}-RS certifies a learned gate under a measured fidelity assumption (Table~\ref{tab:backends}).

\paragraph{Contributions.}
\begin{itemize}
    \item \textbf{Robust authorization problem.} We formalize authorization after a tool return as a decision under bounded semantic uncertainty: the runtime observes a validated record \(\tilde z\), but the correctly bound return may be any \(z^\star\in\Ball(\tilde z)\). Proposition~\ref{prop:return} establishes that return-dependent safety requires reading the realized return (\(\approx 44\%\) of matched pairs, Section~\ref{sec:q2}); \method{} replaces pointwise evaluation with certification over the full joint neighborhood (Figure~\ref{fig:eval-locus}).
    \item \textbf{Non-composition.} We prove separate certificates for the categorical and numerical channels do not imply safety over their Cartesian product (Theorem~\ref{thm:noncomp}): joint-gap witnesses are safe at the observed point and under each marginal perturbation, yet unsafe for a combined move; the witness interval has length \(\min(\Delta,\varepsilon)\) and the empirical witness frequency follows this geometry (Section~\ref{sec:q3}).
    \item \textbf{Certified monitor with an assumption ladder.} \method{} enumerates \(\Ndisc(s)\) exactly and certifies each branch by a backend from an assumption ladder, with a backend-independent abstention floor (Proposition~\ref{prop:floor}). For the learned rungs, the certificate is over the gate; it is policy-sound under an explicit gate--policy fidelity condition, measured throughout the evaluation.
    \item \textbf{Measured safety case.} We calibrate the discrete and continuous budgets from injected return-assembly faults, derive a deployment-specific freshness limit, and evaluate adaptive attacks and policy shift. The formal guarantee is separated from its operational preconditions and residual risk; external-validity claims concern existence, mechanism, and frequency within the evaluated settings, and field prevalence is not measured.
\end{itemize}

\section{Background and Related Work}

The harness owns the agent loop's security-relevant boundaries, the raw model reduced to a proposer~\cite{yao2023react,liu2026claudecode} (extended background: Appendix~B).
Deployed permission systems sit at those boundaries but are almost entirely \emph{pre-execution}; a recent stress test reports the same state-changing effect reachable through an unevaluated path~\cite{ji2026permissiongate}.
\emph{Indirect prompt injection} and tool/description poisoning~\cite{greshake2023indirect,zhan2024injecagent,debenedetti2024agentdojo,wang2025mcptox} are a \emph{different stage}: there adversarial bytes steer the model, whereas we assume a tool has already returned a typed object and ask whether acting on it is authorized.

\paragraph{Certified robustness.}
Two certification families transfer to typed authorization.
\emph{Lipschitz-margin certification} certifies any decision whose margin exceeds \(L\varepsilon\)~\cite{tsuzuku2018lipschitz}, with expressive \(1\)-Lipschitz architectures built from orthogonal layers~\cite{anil2019sorting,trockman2021orthogonalizing}; it is our primary learned backend (\method{}-Lip, Section~\ref{sec:backends}).
\emph{Randomized smoothing} gives model-agnostic guarantees under continuous perturbations for black-box classifiers~\cite{cohen2019certified}, with extensions to conformal prediction~\cite{zargarbashi2025onesample}; it backs \method{}-RS.
Enumerate-then-certify parallels certified word-substitution defenses~\cite{jia2019certified,huang2019achieving}.
Smoothing an LLM \emph{transcript} classifier is infeasible (high-dimensional, discrete, semantically unstable); typed returns are favorable: low-dimensional continuous fields, small \(\Ndisc(s)\) (\(8\)--\(18\) at \(d=1\), Table~S29).
Hybrid and discrete smoothing kernels~\cite{lee2019tight,bojchevski2020efficient} pay discrete Monte-Carlo variance and win only when \(\Ndisc(s)\) is too large to enumerate \emph{and} the gate is black-box; neither holds here (Proposition~4).
The technical components build on established certification methods; the contribution lies in the joint-neighborhood authorization formulation, the non-composition result, and the measured safety case.

\paragraph{Runtime enforcement, verification, and what stays uncertified.}
\begin{figure}[t]
    \centering
    \includegraphics[width=\linewidth]{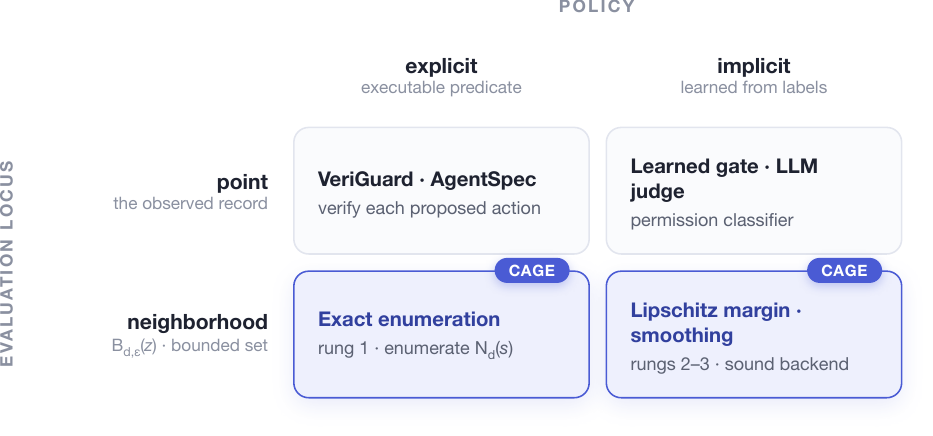}
    \caption{Two axes place the post-tool-return node: \emph{evaluation locus} (the observed point vs.\ the bounded neighborhood \(\Ball(z)\)) and \emph{policy} (executable predicate vs.\ learned from labels); our monitor occupies the neighborhood row, off-axis defenses the text channel, tool descriptions, or training data.}
    \label{fig:eval-locus}
\end{figure}

Two axes separate our node from recent agent-security work (Figure~\ref{fig:eval-locus}): whether a defense reads the typed pair \((z,a)\), and whether it is \emph{certified} under bounded uncertainty between the validated record and its correctly bound return.
A growing line of \emph{runtime enforcement} reads the action and checks it against a policy but offers no uniform guarantee over bounded semantic return uncertainty: VeriGuard formally verifies a synthesized policy and runs a per-action monitor~\cite{miculicich2025veriguard}; AgentSpec enforces trigger/predicate/enforcement rules with no trajectorial analysis~\cite{wang2026agentspec}; PFI and SAFEFLOW import OS-style isolation and information-flow control~\cite{kim2025pfi,li2025safeflow}.
These runtime defenses evaluate the observed action--state pair and do not provide a certificate uniform over \(\Ball(z)\); \method{} provides this uniform neighborhood certificate, while the runtime placement follows this enforcement line.
The certified line targets \emph{different} corruptions, adversarial tool \emph{descriptions} (ToolCert~\cite{yeon2025toolcert}) and \emph{training-set} poisoning (TPA~\cite{ghitu2026tpa}); neither addresses certification over \(\Ball(z)\) at inference, an open surface in a recent SoK~\cite{ling2026securesok}.
Beyond agent security, access-control policy verification~\cite{fisler2005margrave}, attribute-based access control~\cite{hu2014abac}, runtime shielding~\cite{alshiekh2018shielding}, and robust optimization~\cite{bental2009robust} verify, enforce, or optimize the policy itself; \method{} certifies the authorization \emph{decision} over \(\Ball(z)\) at inference.
Our safety predicate follows the oracle-function formalization of agent security~\cite{siu2026formalizing}, realized with a typed gate, since LLM judges measure adversarial robustness no better than chance under shift~\cite{schwinn2026coinflip}.

\section{Problem Setting}

\smallskip
\noindent\fcolorbox{gray!45}{examplebg}{\parbox{0.94\columnwidth}{\footnotesize
        \textbf{What exactly is the task?} Given a validated typed tool return \(\tilde z\) and an action \(a\) proposed by the agent, decide whether \(a\) remains authorized for \emph{every} plausible correctly bound return \(z'\in\Ball(\tilde z)\).
        \\[3pt]
        \textbf{Observed return:} \(\texttt{source}=\texttt{credit\_check}\), \(\texttt{score}=0.46\)\\
        \textbf{Proposed action:} \(\texttt{approve\_transaction}\)\\
        \textbf{Policy:} the \(\texttt{score}\) threshold depends on the \(\texttt{source}\)\\
        \textbf{Uncertainty:} one source-binding error and bounded score drift}}
\smallskip

\subsection{Typed tool returns}

We model a structured tool return as \(z=(s,x)\), where \(s\in\mathcal S\) is the discrete part (tool identity, provenance, categorical fields, status labels) and \(x\in\mathbb R^k\) the continuous part (risk scores, amounts, latencies, prices).

The agent proposes a candidate downstream action \(a\in\mathcal A\); the policy oracle is an action-indexed predicate \(\Safe(z,a)\in\{0,1\}\), and the same return may be safe for \(\texttt{manual\_review}\) yet unsafe for \(\texttt{approve\_transaction}\).

\subsection{Threat model}

The gate reads an observed record \(\tilde z=(s,x)\): typed, schema-validated, integrity- and freshness-checked.
The record is syntactically valid at decision time; residual uncertainty concerns its binding to provenance, fresh values, and policy state.
The state this record binds to (the provenance in force, the fresh score, the resolved policy pack) is denoted \(z^\star\); assembly by adapters, caches, and schema migrations means \(z^\star=(s',x')\) may differ from \(\tilde z\) only within the validated envelope
\[
    \Ball(s,x)
    =
    \left\{
    (s',x'):\Disc(s,s')\le d,\ \|x'-x\|_2\le\varepsilon
    \right\}.
\]

The certified object is \(\Allow(\tilde z,a)=1\Rightarrow\forall z^\star\in\Ball(\tilde z):\ \Safe(z^\star,a)=1\): robust authorization under bounded semantic return uncertainty; hereafter we write \(z\) for the observed record.
The discrete component of \(\Ball\) captures the small admissible \emph{binding} faults of that assembly (a stale provenance tag, a confused policy pack, a read-then-act race).
The continuous component captures the residual numerical uncertainty left after ordinary validation.
Fully fabricated identities, compromised endpoints, and values outside the declared radius are outside the certificate (Table~S32); the typed constructor and validation stack sit in the trusted computing base (TCB; Appendix~C).
The budget is two-layered: integrity+freshness validation \emph{enforces} it against adaptive attackers; \(\varepsilon\) is a p95 residual \emph{calibrated} on the tested fault process (Section~\ref{sec:q6}); \(d=1\) is the measured single-fault granularity (\(d=2\) recommended against fault-inducing adversaries; Tables~S28--S29).

\subsection{Two channels: typed return and untrusted display text}
\label{sec:tm}

In deployment the model consumes \(o=(z,m)\), validated record plus attacker-influenced text: two threat models follow (Figure~\ref{fig:post-return-boundary}). \\
\textbf{TM1 (untrusted text channel):} the adversary controls \(m\) but not \(z\); \(m\) can move the LLM proposal, but a gate reading only \(z\) has \(\Allow(z,a)\) independent of \(m\) \emph{by construction}: the defense is architectural; any typed gate excluding \(m\) inherits it (cross-action mis-selection keeps a \(0.31\) residual, Table~S33). \\
\textbf{TM2 (bounded semantic return uncertainty):} the runtime observes the validated record \(z\), while the correctly bound return may be any \(z^\star\in\Ball(z)\); excluding \(m\) no longer suffices; authorization must hold uniformly over \(\Ball(z)\).

\subsection{Return-dependent authorization requires observing the return}

The requirement is independent of any certificate machinery.

\begin{proposition}[Return dependence]
    \label{prop:return}
    Suppose there exist two possible tool returns \(z_0,z_1\) and a candidate action \(a\) such that \(\Safe(z_0,a)=1\) and \(\Safe(z_1,a)=0\).
    Then any authorization rule that does not inspect the realized return \(z\) cannot correctly authorize \(a\) in both cases.
\end{proposition}

The elementary proof is in Appendix~A. Proposition~\ref{prop:return} locates the decision after the return becomes available, and existing runtime hooks can implement it; Section~\ref{sec:q2} measures how often return dependence occurs.

\begin{table}[t]
    \centering
    \small
    \caption{Notation used throughout.}
    \label{tab:notation}
    \resizebox{\columnwidth}{!}{%
    \begin{tabular}{ll}
        \toprule
        symbol & meaning \\
        \midrule
        \(\tilde z=(s,x)\), later \(z\)          & observed validated record (\(s\) discrete, \(x\) continuous) \\
        \(z^\star\)                              & a plausible correctly bound return represented by the observed record \\
        \(\Ball(z)\)                             & budget: \(\le d\) discrete edits and \(\ell_2\) drift \(\le\varepsilon\) \\
        \(\Ndisc(s)\)                            & discrete neighborhood \(\{s':\Disc(s,s')\le d\}\) \\
        \(U,\;C,\;R\)                            & unsafe at point / joint-gap witness / robust in budget \\
        \(W=\{\text{clean-safe}\}\setminus R\)   & the witnesses any sound gate must refuse \\
        CFA                                      & certified false-allow rate (oracle-measured) \\
        \(\Rallow\)                              & robust-safe allow rate (certified autonomy) \\
        rung 1 / 2 / 3                           & \method{}-Exact / \method{}-Lip / \method{}-RS backend \\
        TM1 / TM2                                & untrusted text channel / bounded return uncertainty \\
        \bottomrule
    \end{tabular}}
\end{table}

\section{Why Marginal Certificates Do Not Compose}

A natural baseline certifies the two channels separately and allows if both marginal certificates pass; this is unsound.
A continuous certificate at the original discrete state proves \(\forall x'\in B_\varepsilon(x):\Safe((s,x'),a)=1\), and a discrete certificate at the original numerical value proves \(\forall s'\in\Ndisc(s):\Safe((s',x),a)=1\), but the needed joint statement is strictly stronger:
\begin{align*}
                      & \bigl[\forall x'\in B_\varepsilon(x):\Safe((s,x'),a)=1\bigr]                            \\
    \land\quad        & \bigl[\forall s'\in\Ndisc(s):\Safe((s',x),a)=1\bigr]                                    \\
    \nRightarrow\quad & \bigl[\forall s'\in\Ndisc(s),\,\forall x'\in B_\varepsilon(x):\Safe((s',x'),a)=1\bigr].
\end{align*}
The joint-gap attack occupies exactly this logical gap: the unsafe region exists only in the discrete--continuous interaction.

\begin{theorem}[Non-composition of marginal certificates]
    \label{thm:noncomp}
    There exist safety predicates \(\Safe\), points \((s,x)\), actions \(a\), discrete neighborhoods \(\Ndisc(s)\), and continuous balls \(B_\varepsilon(x)\) such that \(\Safe((s,x'),a)=1\) for all \(x'\in B_\varepsilon(x)\) and \(\Safe((s',x),a)=1\) for all \(s'\in\Ndisc(s)\), yet \(\Safe((s',x'),a)=0\) for some \(s'\in\Ndisc(s)\), \(x'\in B_\varepsilon(x)\).
\end{theorem}

With one-dimensional \(x\) and a threshold that shifts by \(\Delta\) across the swap, continuous-only safety needs \(x\le q-\varepsilon\) and discrete-only safety needs \(x\le q-\Delta\), yet the joint move is unsafe for \(x>q-\Delta-\varepsilon\); all three hold exactly on the witness interval \(\bigl(q-\Delta-\varepsilon,\ \min(q-\varepsilon,q-\Delta)\bigr]\) of length \(\min(\Delta,\varepsilon)\) (Figure~\ref{fig:joint-gap-witness}; proof in Appendix~A).
The prevalence of \(C\) therefore scales as \(\min(\Delta,\varepsilon)\) (measured in Section~\ref{sec:q3}).

\section{Method}

\method{} certifies the actual joint perturbation set: enumerate the finite discrete neighborhood \(\Ndisc(s)=\{s':\Disc(s,s')\le d\}\) exactly, run a sound per-branch test on the continuous ball at every \(s'\), and allow only when every branch passes,
\[
    \Allow_{\mathrm{\method{}}}(s,x,a)=1 \iff \min_{s'\in\Ndisc(s)}\mathrm{Cert}_\varepsilon(s',x,a)=1 .
\]
A certified allow covers every branch under the continuous budget, up to the branch test's stated confidence: exactly the joint budget the joint-gap attack exploits.
Each branch is read as a plausible correctly bound return \(z^\star\in\Ball(z)\) (TM2, Section~\ref{sec:tm}).
The branch test \(\mathrm{Cert}_\varepsilon\) comes from the assumption ladder of Section~\ref{sec:backends}.

\begin{table*}[t]
    \centering
    \small
    \caption{What each backend certifies: rung 1 is \emph{policy-certified}; rungs 2--3 and the ceiling are \emph{gate-certified}, policy-sound under measured fidelity. Every result is one of: policy certificate / gate certificate / oracle-measured fidelity (\(\texttt{cert\_false\_allow}\), henceforth CFA) / out-of-budget escape (Section~\ref{sec:q6}).}
    \label{tab:backends}
    \resizebox{\textwidth}{!}{%
        \begin{tabular}{lllll}
            \toprule
            backend                           & used when                & certifies                  & cost / slack                                  & failure mode             \\
            \midrule
            \method{}-Exact (rung 1)          & affine-fragment policy   & the policy \(\Safe\)       & none inside \(\Ball\)                         & TCB / budget escape      \\
            \method{}-Lip (rung 2)            & trainable gate           & trained gate \(h_\theta\)  & \(L_{\mathrm{cert}}\varepsilon\) margin slack & gate--policy misfit      \\
            \method{}-RS (rung 3)             & black-box gate           & smoothed gate \(p_\theta\) & \(\sigma\Phi^{-1}(\tau)\) buffer + MC         & misfit + MC conservatism \\
            complete verif.\ / MILP (ceiling) & small ReLU gate, offline & trained gate, exactly      & seconds per record                            & misfit passed through    \\
            \bottomrule
        \end{tabular}}
\end{table*}

\textbf{Worked example.} For the screening example, take \(q=0.60\), \(\Delta=0.08\), \(\varepsilon=0.10\); the observed record \((\texttt{credit\_check},0.46)\) is then a joint-gap witness: \(0.46\le q-\varepsilon\) and \(0.46\le q-\Delta\), yet \(0.46>q-\Delta-\varepsilon\). \method{}-Exact enumerates \(\Ndisc(s)=\{\texttt{credit\_check},\texttt{sanctions\_screen}\}\) and runs the rung-1 test on each branch: \(0.46+0.10=0.56\le0.60\) certifies the \texttt{credit\_check} branch, but \(0.56>0.52\) fails the \texttt{sanctions\_screen} branch. The minimum over branches is \(0\): the action is therefore refused and routed to review, consistent with Proposition~\ref{prop:floor}.

\subsection{Certification backends: an assumption ladder}
\label{sec:backends}

The monitor's discrete channel is \emph{always} exact enumeration over the joint budget \(\Ball(z)\); only the continuous branch test varies along an \emph{assumption ladder} (Table~\ref{tab:backends}): each rung relaxes a requirement on the gate and pays a characterized slack over the soundness floor.
When the policy is implicit we train a binary action-indexed gate on typed returns and oracle labels: a \(1\)-Lipschitz signed-margin network \(h_\theta\) (orthogonal layers~\cite{boissin2025orthogonium}, \(L_{\mathrm{cert}}=1\)) for rung 2, or Gaussian smoothing of a black-box gate for rung 3 (Appendix~D).
Rung 1 allows iff every affine constraint of every branch passes the exact test \(w^\top x+\varepsilon\|w\|_2\le b(s',a)\), in \(O(|\Ndisc(s)|\,mk)\) over a \emph{verified affine fragment} (Definition~1, Proposition~7).
Rung 2 allows iff \(\min_{s'\in\Ndisc(s)}h_\theta(s',x,a)>L_{\mathrm{cert}}\varepsilon\); rung 3 iff \(\min_{s'\in\Ndisc(s)}\underline p(s',x,a)\ge\tau\).
An out-of-fragment policy is detected syntactically and routed to the ceiling or a gate-certified rung, never silently approximated.
We reserve \emph{policy-certified} for rung 1 and call rungs 2--3 (and the ceiling) \emph{gate-certified}.

\method{}-Exact applies when the policy is executable. \method{}-Lip is the default learned backend for implicit policies, while \method{}-RS provides a conservative fallback for black-box gates and bounded batches.
Certification of a learned gate does not establish policy correctness; gate--policy fidelity is evaluated separately.
\method{}-RS's smoothed certificate reduces empirical policy exposure without removing the fidelity assumption, and its per-decision confidence fails over a deployment horizon (Table~S47).
The complete verification / MILP ceiling (CROWN-style verifier~\cite{zhang2018crown} when available; our reported run uses MILP~\cite{tjeng2019evaluating}) is the offline audit rung: it certifies the trained gate's \emph{entire} robust-safe set, identifying whether a low allow rate stems from the learned margin or from the certification backend (Table~S6).
Where \(\Safe\) is executable, rung 1 dominates and the learned rungs serve as ablations; their practical role is the implicit-policy regime, where an executable predicate does not exist by construction (Section~\ref{sec:q5}).

\subsection{Guarantees}
\label{sec:theory}

The guarantees rely on three properties (proofs in Appendix~A); define the worst-branch joint margin \(m(z,a)=\min_{s'\in\Ndisc(s)}\inf\{\|x'-x\|_2:\Safe((s',x'),a)=0\}\), so the robust-safe set is \(R=\{m>\varepsilon\}\).

\begin{itemize}
    \item Exact enumeration with per-branch bounds is a valid \emph{joint} discrete--continuous certificate, with no Monte-Carlo error on the discrete channel (Proposition~4).
    \item The Lipschitz margin test certifies the learned gate's positive decision over the whole joint ball (Proposition~\ref{prop:lip}).
    \item Any sound gate must abstain near the unsafe boundary (Proposition~\ref{prop:floor}).
\end{itemize}

\begin{proposition}[Lipschitz-margin certificate]
    \label{prop:lip}
    If \(\Allow_{\mathrm{Lip}}(z,a)=1\), then \(h_\theta(s',x',a)>0\) for every \((s',x')\in\Ball(z)\): the learned gate's positive decision is certified over the whole joint budget.
\end{proposition}
\begin{proposition}[Soundness requires boundary abstention]
    \label{prop:floor}
    If a gate is \emph{sound} for \(\Ball\) (\(\Allow(z,a)=1\Rightarrow\Safe(z',a)=1\) for all \(z'\in\Ball(z)\)), then \(\Allow(z,a)=1\Rightarrow m(z,a)\ge\varepsilon\), strict when the unsafe set is closed; hence \(\Pr[\Allow=1]\le\Pr[m\ge\varepsilon]\), with equality for the oracle robust gate up to the boundary \(\{m=\varepsilon\}\) (empirical mass \(\approx10^{-5}\); Appendix~A). Every joint-gap witness, and every point strictly within worst-branch margin \(\varepsilon\) of the unsafe set, is refused by \emph{any} sound mechanism.
\end{proposition}
\begin{figure}[t]
    \centering
    \includegraphics[width=\linewidth]{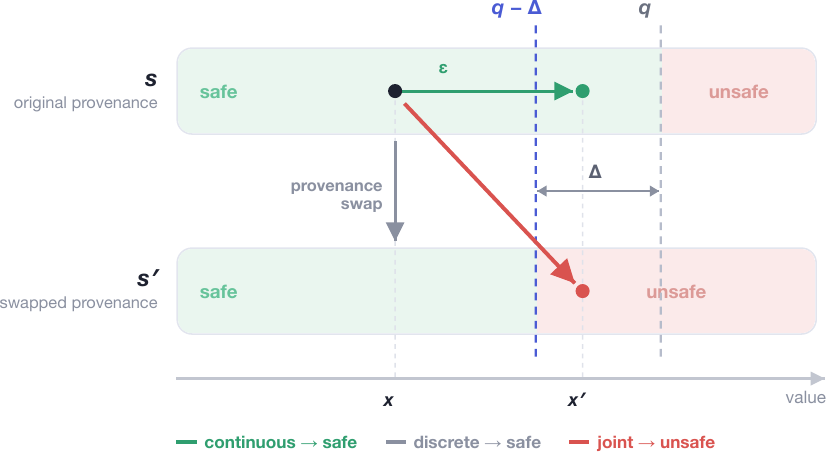}
    \caption{Geometry of a joint-gap witness: safe at the point and under each marginal check, yet the joint move (provenance swap + continuous shift) crosses the shifted threshold \(q-\Delta\); every such point lies within \(\varepsilon\) of the unsafe set, so any sound gate must refuse it (Proposition~\ref{prop:floor}).}
    \label{fig:joint-gap-witness}
\end{figure}

This backend-independent floor unifies the ladder: rung 1 attains it, rungs 2--3 approach it by the slack column of Table~\ref{tab:backends}; CFA measures the gate-overestimation failure (Propositions~5--6, Table~S6).

\section{Experiments}
\label{sec:experiments}

The evaluation first establishes the existence of joint-gap witnesses (Tables~\ref{tab:headtohead}--\ref{tab:systems}), then measures the soundness--autonomy trade-off (Table~\ref{tab:soundness}), and finally tests whether the logical gap produces committed side effects (\S\ref{sec:e2e}).
Additional experiments test whether these results depend on judge capacity, the analytic oracle, the attack strategy, policy stationarity, or the assumed budget (\(d,\varepsilon\), freshness).
The certificate covers a single decision \(\Allow(z,a)\) under the declared budget (\(d{=}1\), \(\varepsilon{=}0.10\); \S\ref{sec:q6}); rung~1 is policy-certified, rungs~2--3 are gate-certified under measured fidelity, and the LLM remains an uncertified proposer. Section~\ref{sec:limitations} discusses the assumptions outside this guarantee.

\begin{table*}[t]
    \centering
    \small
    \caption{Experiment map. Each setting is named by its input data, its policy or oracle, its execution substrate, and its scientific role, so that ``real'' is never ambiguous between real data, a real engine, a third-party policy, and real execution.}
    \label{tab:expmap}
    \resizebox{\textwidth}{!}{%
        \begin{tabular}{lllll}
            \toprule
            setting                    & input data                 & policy or oracle        & execution substrate & role               \\
            \midrule
            synthetic finance/SRE/ops  & generated records          & controlled oracle       & simulator           & geometry           \\
            IEEE-CIS + OPA             & real transaction marginals & authored OPA policy     & OPA                 & primary comparison \\
            OpenFisca                  & generated boundary cases   & third-party policy code & OpenFisca           & external validity  \\
            k8s / MCP / Marble         & constructed episodes       & instrumented policies   & live systems / APIs & committed effects  \\
            fault injection            & corrupted adapter returns  & corresponding oracle    & adapter stack       & budget calibration \\
            \bottomrule
        \end{tabular}}
\end{table*}

\subsection{Setup}
\label{sec:setup}
Table~\ref{tab:expmap} maps each setting to its data, policy, and role; Table~\ref{tab:setup} fixes the shared frame; secondary sweeps live in Appendix~E.
The oracle stratifies each pair by where in \(\Ball(z)\) safety can flip. We focus on three categories: \(U\), unsafe already at the observed point; \(C\), safe under each marginal check yet unsafe under the joint move (Theorem~\ref{thm:noncomp}'s witnesses); and \(R\), robust everywhere in the budget. Category \(C\) tests soundness; category \(R\) tests non-vacuity. We use \(A\)/\(B\) for stratification only.
We report \(\Callow,\Uallow,\Rallow=\Pr[\Allow=1\mid C,U,R]\) and the oracle-measured CFA, a certified allow the oracle deems joint-unsafe. RS parameters are fixed across experiments (Appendix~D).

\begin{table*}[t]
    \centering
    \small
    \caption{Shared experimental frame. Sampling modes: natural, boundary-balanced (tagged stress tests), \(C\)-targeted (Table~S1).}
    \label{tab:setup}
    \resizebox{\textwidth}{!}{%
        \begin{tabular}{ll}
            \toprule
            operating point & \(d=1\), \(\varepsilon=0.10\) (measured and validated in \S\ref{sec:q6})                                                 \\
            baselines       & no gate; learned point gate; exact predicate at the point; marginal composition; oracle                                  \\
            backends        & \method{}-Exact / -Lip / -RS; MILP ceiling (offline); gates: \(1\)-Lipschitz orthogonal (Lip), tabular MLP smoothed (RS) \\
            settings        & synthetic finance/SRE/ops; OPA policy-as-code~\cite{opa2024}; IEEE-CIS; NAB; PSD2/AML; DevOps                            \\
            \bottomrule
        \end{tabular}}
\end{table*}

\subsection{Return dependence and pointwise authorization}
\label{sec:q2}\label{sec:q3}
\textbf{Question.} \emph{How often does authorization depend on the return, and how often do point checks admit joint-gap witnesses?}\\
\textbf{Setup.} Matched pairs sharing domain, action, and schema (\(10{,}000\)/domain, analytic oracle; generator-conditional rates).\\
\textbf{Result.} Return-dependent pairs span \(42.5\)--\(45.6\%\) across finance/SRE/ops, so any return-blind rule errs on the minority class (majority error \(31.6\)--\(34.0\%\)). A learned pre-return baseline errs \(28\)--\(33\%\), a post-return gate \(0.4\)--\(0.7\%\).\\
\textbf{Interpretation.} A pre-return rule cannot distinguish these cases; a post-return gate can (Proposition~\ref{prop:return}).

\begin{table*}[t]
    \centering
    \small
    \caption{Point vs.\ neighborhood on the Proposition~\ref{prop:floor} witness set \(W=\{\text{clean-safe}\}\setminus R\) (real IEEE-CIS + OPA policy; \(6032/10{,}000\)). CFA is the joint-unsafe fraction of \emph{all} certified allows in the run, not of \(W\). \(^\dagger\)\method{}-Lip is gate-certified; the nonzero entry is a base-gate fidelity caveat.}
    \label{tab:headtohead}
    \begin{tabular}{llcc}
        \toprule
        gate                                             & evaluates    & allow-on-\(W\)  & CFA \\
        \midrule
        no gate & --- & \(100\%\)~(by constr.) & ---                             \\
        learned point-gate & point & \(99.3\%\)~(by constr.) & ---                             \\
        exact predicate @ point (OPA) & point & \(100\%\)~(by constr.) & ---                             \\
        LLM point judge (32B)                            & point        & \(92\%\)        & ---                             \\
        budget-aware LLM judge                           & point        & \(44\pm43\%\)   & ---                             \\
        \method{}-Exact (\(d{=}1\) enum.) / \method{}-RS & neighborhood & \(0\%\)         & \(0\)                           \\
        \method{}-Lip                                    & neighborhood & \(1.5\pm0.4\%\) & \(8.2\%^\dagger\)               \\
        \bottomrule
    \end{tabular}
\end{table*}

Because every element of \(W\) is safe at the observed point by construction, point accuracy is not informative on this set. Table~\ref{tab:headtohead} instead distinguishes gates that evaluate the observed point from gates that certify the full neighborhood \(\Ball(z)\).

\textbf{Frequency of joint-gap witnesses.} Theorem~\ref{thm:noncomp}'s geometry is confirmed empirically: sweeping the authored gap \(\delta\) on both real datasets, \(\Pr[C]\) tracks \(\min(\delta,\varepsilon)\) (correlation \(1.0\)/\(0.96\); Table~S44).
Joint-gap witnesses follow the geometry of a provenance-shifted threshold and arise in every setting we measured, at \(3.5\)--\(12\%\) natural within-setting frequency (Table~\ref{tab:soundness}); DevOps is boundary-constructed (Table~S9).
OPA and GoRules ZEN~\cite{goruleszen} reproduce the analytic taxonomy exactly (agreement \(1.000\); Table~S24) and Marble~\cite{checkmarble2024marble} yields \(C\)-set Jaccard \(1.000\) (Table~S21), and a second real adapter reproduces the taxonomy through Kyverno admission (Table~S23), ruling out an oracle artifact.

\textbf{Point and marginal certificates fail.} Neither pointwise exactness nor channel-wise composition suffices: on the witness set \(W=\{\text{clean-safe}\}\setminus R\), which any sound gate must refuse (Proposition~\ref{prop:floor}), the exact OPA predicate allows \emph{all} witnesses, and every other point-evaluating gate (learned, LLM judges, guard model~\cite{inan2023llamaguard}) behaves the same, while every neighborhood-evaluating gate refuses \(W\) with \(\mathrm{CFA}=0\) (Tables~\ref{tab:headtohead}, S17).
A stronger \(36\)B-MoE judge does not close the gap (\(\le61\%\) of \(W\) allowed, unchanged at temperature \(0\)): the results are explained by the evaluation locus: added judge capacity does not provide a neighborhood certificate.
By their published evaluation semantics, deployed enforcers AgentSpec and VeriGuard evaluate only the observed \((z,a)\) and admit every witness by construction (Table~S18).

\subsection{End-to-end agent and system validation}
\label{sec:e2e}
\textbf{Question.} \emph{Do point and marginal failures produce concrete unsafe effects in a running agent, and does \method{} stop them?}\\
\textbf{Setup.} Full chains with committed side effects: finance/SRE money-commit and alert-suppression episodes, a live k8s/Kyverno cluster on a stale-registry TOCTOU surface with a real MCP write path, and Marble's decision API.\\
\textbf{Result.} Table~\ref{tab:systems}: marginal composition and a transcript classifier with point accuracy \(1.00\) commit at the full no-gate rate, pre-execution screening at \(0.44/0.23\) (Table~S19); only the joint certificate and the oracle commit none on every substrate, safe controls non-vacuous.
CaMeL and the joint certificate are orthogonal: each blocks a case the other admits (Table~S18).
Gate overhead is \(\approx10\,\mu\)s vs \(\approx2.5\)\,s LLM decode; protocols: Tables~S18 (AgentSpec/VeriGuard/CaMeL), S20 (k8s/Kyverno), S22 (Marble).

\begin{table}[t]
    \centering
    \small
    \caption{End-to-end validation on deployed substrates: \method{} blocks the unsafe effect while the matching safe control passes (\(\checkmark\); the MCP-write arm has no safe-control counterpart).}
    \label{tab:systems}
    \setlength{\tabcolsep}{2.5pt}
    \begin{tabular}{lll}
        \toprule
        substrate               & ungated unsafe effect                & \method{} \\
        \midrule
        fin/SRE                 & \(1.00/0.40\) (=point=marginal)      & \(0/0\)\,\(\checkmark\)   \\
        k8s stale registry      & unsafe deploy                        & blocked\,\(\checkmark\)   \\
        MCP write (4 LLMs)      & side effect \(4/4\)                  & \(0/4\)   \\
        filesystem MCP quota    & over-quota write                     & blocked\,\(\checkmark\)   \\
        Marble decision API     & \(100/100\) persisted                & \(0/100\)\,\(\checkmark\) \\
        \bottomrule
    \end{tabular}
\end{table}

\subsection{Certificate soundness and autonomy on real settings}
\label{sec:q5}
\textbf{Question.} \emph{Is every certified allow oracle-safe, while the gate still allows a useful fraction of the robust-safe set?}\\
\textbf{Setup.} Primary setting: executable policy-as-code (labels from a real \texttt{opa eval} engine) over real IEEE-CIS marginals; each Table~\ref{tab:soundness} row carries its backend tag.\\

\begin{table}[t]
    \centering
    \small
    \setlength{\tabcolsep}{3pt}
    \caption{Soundness and robust-safe autonomy across the evaluated settings at \(d{=}1\), \(\varepsilon{=}0.10\): every certified allow is oracle-safe (\(\Callow{=}\Uallow{=}\mathrm{CFA}{=}0\)) with \(\Rallow>0\) throughout; zeros carry Wilson-95\% upper bounds (Table~S51). Per-setting sources: Tables~S27 (OpenFisca), S3 (OPA), S1 (IEEE-CIS), S2 (NAB), S9 (PSD2/AML).}
    \label{tab:soundness}
    \begin{tabular}{llccc}
        \toprule
        setting       & backend & \(C\) freq.        & CFA   & \(\Rallow\)                         \\
        \midrule
        OpenFisca     & Exact   & \(3.5\)--\(8.9\%\) & \(0\) & \(22\)--\(34\%\)                    \\
        \midrule
        OPA           & Lip/RS  & \(10\)--\(12\%\)   & \(0\) & \(27\)--\(37\%\)/\(6.5\)--\(7.3\%\) \\
        IEEE-CIS      & RS      & \(6.4\%\)          & \(0\) & \(20\)--\(24\%\)                    \\
        NAB telemetry & Lip     & \(10.9\%\)         & \(0\) & \(100\%\)                           \\
        PSD2/AML      & RS      & \(6.5\)--\(9.8\%\) & \(0\) & \(20\)--\(34\%\)                    \\
        \bottomrule
    \end{tabular}
\end{table}

In Table~\ref{tab:soundness}, CFA measures soundness, whereas \(\Rallow\) measures the fraction of robust-safe decisions retained for autonomous execution.

\textbf{Result: soundness without vacuity.} Across Table~\ref{tab:soundness}, every certified allow is oracle-safe, while every setting retains nonzero robust-safe coverage.
On IEEE-CIS, no fraud is observed in the autonomous tranche, whereas \(3\)--\(4\%\) of blocked transactions are fraudulent.
This contrast captures the role of abstention: the certificate removes decisions near an unsafe boundary rather than merely improving point accuracy.

The same pattern holds in three additional settings (robust-safe coverage \(11\)--\(38\%\); Table~S9); scaling sweeps reach \(63\%\) while uncertified attacks keep \(90\)--\(98\%\) false allows (Table~S8).
Under known ground truth, \method{}-Lip recovers \(65\%\) of the exact robust-safe set with zero measured policy false allows. The corresponding Wilson upper bound is \(2.2\%\) (Table~S7).
The nonzero entry for \method{}-Lip in Table~\ref{tab:headtohead} results from base-gate policy misfit; the neighborhood certificate remains valid for the learned gate (Table~S6).

\textbf{Adaptive attacks and policy shift.} An adaptive adversary knowing \(g_\theta,\Ball,\tau,\Safe\) keeps \(1.7\)--\(15\%\) false allows through a point gate (all search families); \method{} removes them all (Table~S10); adaptive display-channel attacks are likewise removed (Table~S12).
Under policy shift only the uncertified component breaks: learned adaptive-FA reaches \(\mathbf{47.5\%}\) on a held-out threshold, certified \(0\%\) (Table~S11).

\textbf{Implicit policy: a gate-fidelity stress test.} With the real IEEE-CIS \(\texttt{isFraud}\) label as policy (outside the affine fragment), only gate-level certification is available: Lip certifies \(72\%\) of safe transactions, RS \(40\%\); under attack a utility-matched point gate re-admits \(\approx48\%\) of fraud vs \(41\%/4.4\%\) certified (Table~S13).

\textbf{Ablations.} Soundness is invariant on every grid tested (\(\varepsilon\times\tau\), \(\sigma\), engine-labelled, backend-swap; Tables~S3--S5); only utility and cost move; the MILP ceiling certifies \(\Rallow=100\%\), locating the deficit in learned margin (Table~S6).

\textbf{External validity.} A deployed third-party rule carries the threshold idiom itself: OpenFisca-France's social-housing income ceilings vary by zone~\cite{openfisca}, with gaps of EUR~5,029--34,543; the point gate grants all \(1{,}773\) joint-gap witnesses, the exact certificate none (Table~\ref{tab:soundness}; Table~S27).
The same structure is first-class in DMN decision tables and policy code~\cite{gatekeeperlibrary2024,azurekeyvaultpolicy}, and in the AML engines Tazama~\cite{tazama2024} and Marble~\cite{checkmarble2024marble} (Table~S26).
The PSD2/AML thresholds~\cite{psd2rts2018,fincenctr} instantiate the continuous mechanism with source-locked constants (Tables~S9, S24).
What remains unmeasured is prevalence in deployed agentic pipelines: three protocol-frozen scans and a registry adjudication return pipeline-provenance nulls (\(0/31\) documented \(\theta(s)\); Tables~S24, S25).
We therefore claim existence, realizability, and mechanism; field prevalence remains unmeasured.
With four real proposers (three Qwen2.5 variants and the 36B-MoE \texttt{qwen3.6}), the measured certified false-allow rate remains zero for every proposer and domain (Tables~S14--S16).

\textbf{Takeaway.} In the implicit-policy setting, performance is evaluated relative to the soundness limit of Proposition~\ref{prop:floor}.
At \(37\%\) certified autonomy on IEEE-CIS, in-budget fraud in the autonomous tranche is \(4.4\%\), versus \(53.2\%\) for a volume-matched point gate (Table~S49). On natural traffic, \method{}-Exact autonomously clears \(57\%\) and \(46\%\) of IEEE-CIS and NAB decisions, respectively, whereas the stricter learned backends clear \(2\)--\(14\%\) (Table~S50).

\begin{figure}[t]
    \centering
    \includegraphics[width=0.88\linewidth]{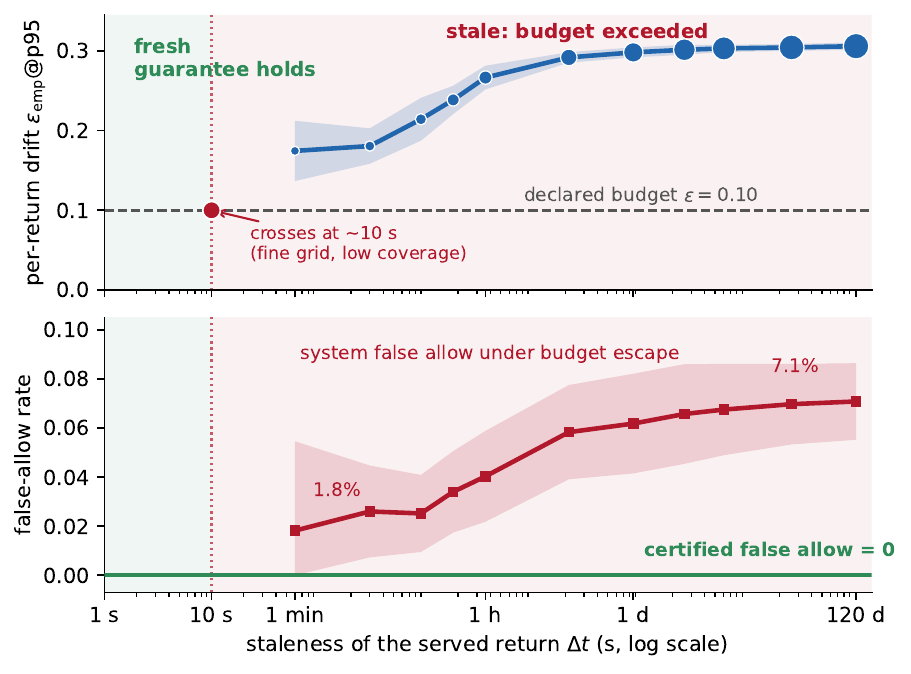}
    \caption{Freshness governs the residual. \emph{Top:} same-entity wall-clock \(\varepsilon_{\mathrm{emp}}@\text{p95}\) crosses the declared \(\varepsilon=0.10\) (dashed) at \(\approx10\)\,s staleness (Table~S35). \emph{Bottom:} the system false-allow grows as the SLA relaxes while certified false-allow stays \(0\): residual risk is bounded by the measured escape.}
    \label{fig:reconcile}
\end{figure}

\subsection{Assumptions and residual risk}
\label{sec:q6}
\textbf{Question.} \emph{Does the declared budget \((d{=}1, \varepsilon{=}0.10)\) match how real return corruption behaves, and what escapes it?}\\
\textbf{Setup.} The budget is an empirical claim about return-assembly failure: seven adapter faults give per-fault \((d,\varepsilon)\), replayed timestamps test freshness, leave-one-out on \(\Ndisc\) tests coverage, and constructor corruption probes the TCB.\\
\textbf{Result.} Table~\ref{tab:faults}: every fault the budget claims to cover is atomic (\(d=1\)) with residual inside \(\varepsilon=0.10\); the two uncovered faults have no discrete footprint and blow past \(\varepsilon\) (validation-stack failures); corrupting the trusted constructor raises verified-point false-allow to \(1.7\%\) (Table~S34); recalibration through a real adapter into Marble holds held-out, leaking only outside the ball (Tables~S33, S39).

\begin{table}[t]
    \centering
    \small
    \caption{Budget measurement from the injected adapter-fault mechanisms (\(4000\)/fault; catalogue in Appendix~C). Covered faults are atomic (\(\Pr[d{=}1]=1.000\)) and the validated residual is \(\varepsilon_{\mathrm{emp}}@\mathrm{p95}\approx0.079\)--\(0.095\) (per-fault: Table~S31); leave-one-out escape is \(0\) for every covered class; compound faults land at \(d\le2\), still sound, a rare triple at \(d=3\) (Table~S28); branch over-declaration in Table~S30.}
    \label{tab:faults}
    \resizebox{\columnwidth}{!}{%
        \begin{tabular}{llll}
            \toprule
            fault class            & discrete covered & continuous residual     & escape mode        \\
            \midrule
            stale provenance       & yes (atomic)     & within \(\varepsilon\)  & none               \\
            policy-pack confusion  & yes (atomic)     & within \(\varepsilon\)  & none               \\
            read-then-act (TOCTOU) & yes (atomic)     & within \(\varepsilon\)  & none               \\
            schema transposition   & no footprint     & exceeds \(\varepsilon\) & validation failure \\
            key collision          & no footprint     & exceeds \(\varepsilon\) & validation failure \\
            \bottomrule
        \end{tabular}}
\end{table}

\textbf{Freshness is the binding precondition.} Same-card staleness crosses \(\varepsilon\) at \(\approx10\)\,s; the system false-allow grows \(1.8\%\to7.1\%\) with the SLA while \(\mathrm{CFA}=0\), so \(\Pr[\text{system false-allow}]\le\Pr[\text{escape}]\) (Figure~\ref{fig:reconcile}; Tables~S35--S36); no well-covered sub-minute SLA meets \(<1\%\): recompute \emph{in-loop}.\\
\textbf{Takeaway: a deployment recipe.} Inject the stack's adapter faults to measure \((d,\varepsilon_{\mathrm{emp}})\); declare \(\Ndisc\), verify leave-one-out escape \(0\); set per-field \(\varepsilon\) and the SLA so \(\varepsilon_{\mathrm{emp}}@\text{p95}\le\varepsilon\) (Tables~S37--S38); audit certified allows; refuse deployment on failed Table~\ref{tab:backends} diagnostics.

\textbf{Dimension and cost.} The regime is low-dimensional typed returns, sound through \(k=50\); above it we project onto the typed policy state (Tables~S46, S40). Certified cost: exact \(<1\,\mu\)s, Lip \(7\)--\(17\)\,ms, RS \(10\)--\(132\)\,ms (\(1.1\)--\(6.5\%\) of decode; Tables~S6, S42); per-field budgets and extensions: Tables~S29/S38/S47.

\section{Discussion, Scope, and Limitations}
\label{sec:limitations}

\method{} formulates authorization over typed tool returns as a joint-neighborhood decision.
The non-composition result identifies pairs that are safe under each marginal check yet unsafe under a combined binding and numerical deviation.
Across policy-engine, regulatory, and real-transaction settings, the joint certificates remove measured in-budget policy false allows while retaining useful autonomy; \method{}-Exact is policy-certified on the verified fragment, \method{}-Lip and \method{}-RS extend it to learned gates under an explicit gate--policy fidelity condition.

The guarantee applies to a single authorization decision and assumes a validated typed-return constructor, complete mediation, an enumerable discrete neighborhood, and continuous uncertainty within the calibrated freshness and numerical budget; budget-escape (\(1.8\%\to7.1\%\), Figure~\ref{fig:reconcile}) and constructor-corruption (\(0\to1.7\%\)) measurements quantify what happens when these conditions weaken.
Cross-turn accumulation, tool selection, prompt-injection resistance, MCP metadata, and multi-step execution remain uncertified; boundary-seeking traffic can inflate abstention but never converts it into a false allow (Table~S41); learned backends require continued gate--policy auditing (false-allow \(\le0.021\) at \(95\%\); Table~S43).

The results support joint-neighborhood authorization as a runtime control wherever typed-return uncertainty can be measured and enforced.
Further work should measure provenance-conditioned policies and binding faults in the field, extend the certificate to sequential decisions, and tie calibration to live validation infrastructure.
\method{} therefore provides a calibrated authorization mechanism for decisions whose safety depends on uncertain typed returns.


\bibliography{references}

\clearpage
\appendix
\setcounter{proposition}{3}
\setcounter{table}{0}
\setcounter{figure}{0}
\renewcommand{\thetable}{S\arabic{table}}
\renewcommand{\thefigure}{S\arabic{figure}}

\noindent These appendices contain the deferred proofs (Appendix~A), extended background (Appendix~B), the fault-mechanism catalogue (Appendix~C), the certification algorithm and experimental setup (Appendix~D), and the additional result tables S1--S51 (Appendix~E).
Proposition and theorem numbering is continuous with the main text: Propositions~1--3 and Theorem~1 are stated in the main text, Propositions~4--7 and Definition~1 here.
Throughout the appendices, \(z\) denotes the validated observed record, following the shorthand introduced in Section~3.2; elements \(z'\in\Ball(z)\) represent plausible correctly bound returns.

\appendix

\section{Deferred proofs}
\label{app:proofs}

\begin{proof}[Proof of Proposition~1 (return dependence)]
    Any pre-return authorization rule observes the user request, transcript, and proposed tool call, but not the realized return.
    Therefore it must assign the same authorization decision to the two executions that differ only in \(z_0\) versus \(z_1\).
    If it allows, it is wrong on \(z_1\); if it denies, it is conservative on \(z_0\).
    Hence return-dependent safety cannot be decided soundly and non-vacuously before the return is observed.
\end{proof}

\begin{proof}[Proof of Theorem~1 (non-composition)]
    Let the continuous variable be one-dimensional and let \(s,s'\) be two discrete states whose thresholds differ by \(\Delta>0\):
    \[
        \begin{aligned}
            \Safe((s,x),a)  & =\mathbf 1[x\le q],        \\
            \Safe((s',x),a) & =\mathbf 1[x\le q-\Delta].
        \end{aligned}
    \]
    Continuous-only safety at \(s\), over the symmetric \(\varepsilon\)-ball, requires \(x+\varepsilon\le q\), i.e.\ \(x\le q-\varepsilon\).
    Discrete-only safety at the observed \(x\) requires \(x\le q-\Delta\).
    But a joint move is unsafe whenever \(x+\varepsilon>q-\Delta\), i.e.\ \(x>q-\Delta-\varepsilon\).
    All three conditions hold simultaneously exactly for
    \[
        x\in\big(\,q-\Delta-\varepsilon,\ \min(q-\varepsilon,\,q-\Delta)\,\big],
    \]
    a nonempty interval of length \(\min(\Delta,\varepsilon)\).
    Any such \(x\) is a joint-gap witness: safe under every continuous-only perturbation at \(s\) and under the discrete swap at the observed \(x\), yet unsafe under the joint perturbation \(s\to s'\), \(x\to x+\varepsilon\).
\end{proof}

\begin{proof}[Proof of Proposition~2 (Lipschitz-margin certificate)]
    Fix a branch \(s'\in\Ndisc(s)\) and any \(x'\) with \(\|x'-x\|_2\le\varepsilon\).
    By \(L_{\mathrm{cert}}\)-Lipschitzness of \(h_\theta\) in the continuous channel,
    \begin{align*}
        h_\theta(s',x',a) & \ge h_\theta(s',x,a)-L_{\mathrm{cert}}\|x'-x\|_2   \\
                          & \ge h_\theta(s',x,a)-L_{\mathrm{cert}}\varepsilon.
    \end{align*}
    If \(\Allow_{\mathrm{Lip}}(z,a)=1\) then \(h_\theta(s',x,a)\ge\min_{s''\in\Ndisc(s)}h_\theta(s'',x,a)>L_{\mathrm{cert}}\varepsilon\), so \(h_\theta(s',x',a)>0\).
    As \((s',x')\) ranges over \(\Ball(z)\), the learned gate stays on the allow side throughout.
\end{proof}

\begin{proof}[Proof of Proposition~3 (soundness requires boundary abstention)]
    Suppose \(m(z,a)<\varepsilon\). Then \(\Ball(z)\) contains a point \(z'=(s',x')\) with \(\Safe(z',a)=0\), and soundness forces \(\Allow(z,a)=0\); contrapositively \(\Allow(z,a)=1\) forces \(m(z,a)\ge\varepsilon\), with no assumption on the unsafe set.
    When the unsafe set is closed the infimum defining \(m\) is attained, so the same argument applies at \(m=\varepsilon\) and the implication sharpens to \(m(z,a)>\varepsilon\).
    Hence \(\Pr[\Allow=1]\le\Pr[m\ge\varepsilon]\); the oracle robust gate allows exactly \(R=\{m>\varepsilon\}\), and \(\Pr[m\ge\varepsilon]=\Pr[m>\varepsilon]\) whenever the boundary \(\{m=\varepsilon\}\) has measure zero (true of our continuous marginals), giving equality.
    A joint-gap witness contains an unsafe point inside \(\Ball(z)\) by definition and is therefore refused, again with no closedness assumption.
\end{proof}

\begin{remark}[Boundary mass on real, quantized marginals]
\label{rem:boundarymass}
Proposition~3's equality clause uses \(\Pr[m=\varepsilon]=0\). On the real pools this holds only approximately: the exact clean safety boundary carries zero empirical mass on IEEE-CIS (\(885{,}809\) record-evaluations) and NAB (\(87{,}556\)); the \(\varepsilon\)-flip knife-edge \(m=\varepsilon\) touches \(7\) and \(1\) records respectively (\(\approx10^{-5}\)); and the one macroscopic atom is the NAB CPU-saturation clip at value \(1.0\) (\(643/87{,}556=0.73\%\)), which lies strictly on the unsafe side and is blocked by both oracles. We therefore read the proposition through its conservative closed inequality: the only effect of the positive boundary mass is a bounded, measured amount of extra abstention (\(\le0.73\%\), dominated by the saturation clip), never a false allow.
\end{remark}

\begin{proposition}[Exactness of enumeration]
    \label{prop:enum}
    Fix \((z,a)\) with \(z=(s,x)\) and \(\Ndisc(s)\) finite. If each per-branch bound \(\underline p(s',x,a)\) is a valid lower confidence bound on \(\inf_{\|x'-x\|_2\le\varepsilon}p_\theta(s',x',a)\) at level \(1-\alpha_{\mathrm{branch}}\) with \(\alpha_{\mathrm{branch}}=\alpha_{\mathrm{FWER}}/|\Ndisc(s)|\), then with probability at least \(1-\alpha_{\mathrm{FWER}}\), \(\underline p_{\mathrm{joint}}(s,x,a)\le\inf_{z'\in\Ball(z)}p_\theta(z',a)\), and the discrete minimization contributes no Monte-Carlo error.
\end{proposition}

\begin{proof}[Proof of Proposition~\ref{prop:enum}]
    By assumption each event \(E_{s'}=\{\underline p(s',x,a)\le\inf_{\|x'-x\|_2\le\varepsilon}p_\theta(s',x',a)\}\) holds with probability at least \(1-\alpha_{\mathrm{branch}}\) with \(\alpha_{\mathrm{branch}}=\alpha_{\mathrm{FWER}}/|\Ndisc(s)|\).
    A union bound over the finitely many branches gives \(\Pr[\bigcap_{s'\in\Ndisc(s)}E_{s'}]\ge 1-\alpha_{\mathrm{FWER}}\).
    On that event,
    \begin{multline*}
        \underline p_{\mathrm{joint}}(s,x,a)=\min_{s'\in\Ndisc(s)}\underline p(s',x,a)\\
        \le\min_{s'\in\Ndisc(s)}\inf_{\|x'-x\|_2\le\varepsilon}p_\theta(s',x',a)=\inf_{z'\in\Ball(z)}p_\theta(z',a),
    \end{multline*}
    where the last equality holds because \(\Ball(z)=\bigcup_{s'\in\Ndisc(s)}\{s'\}\times\{x':\|x'-x\|_2\le\varepsilon\}\).
    The outer minimum is taken over the finite, exactly enumerated set \(\Ndisc(s)\), so it adds no sampling error beyond the per-branch continuous estimates.
\end{proof}

\begin{proposition}[Smoothing buffer and approach to the floor]
    \label{prop:buffer}
    Suppose on each branch \(s'\in\Ndisc(s)\) the base gate realizes the oracle as a halfspace, \(p_\theta(s',x,a)=\Phi(m_{s'}(x)/\sigma)\) with \(m_{s'}\) the signed distance to that branch's boundary. Then in the Monte-Carlo--exact limit the certified gate allows \((z,a)\) iff \(m(z,a)\ge\varepsilon+\sigma\Phi^{-1}(\tau)\); for \(\tau>\tfrac12\) the allow region is contained in \(\{m>\varepsilon\}\) and converges to it as \(\sigma\Phi^{-1}(\tau)\to0\).
\end{proposition}

\begin{proof}[Proof of Proposition~\ref{prop:buffer}]
    In the Monte-Carlo--exact limit the lower confidence bound equals \(p_\theta\), so the per-branch certified bound is \(\underline p(s',x,a)=\Phi\!\bigl(\Phi^{-1}(p_\theta(s',x,a))-\varepsilon/\sigma\bigr)\) and the gate allows iff \(\min_{s'\in\Ndisc(s)}\underline p(s',x,a)\ge\tau\).
    Substituting \(p_\theta(s',x,a)=\Phi(m_{s'}(x)/\sigma)\),
    \begin{align*}
        \underline p(s',x,a)\ge\tau
         & \iff \tfrac{m_{s'}(x)}{\sigma}-\tfrac{\varepsilon}{\sigma}\ge\Phi^{-1}(\tau) \\
         & \iff m_{s'}(x)\ge\varepsilon+\sigma\Phi^{-1}(\tau).
    \end{align*}
    Taking the minimum over branches, the joint rule allows iff \(m(z,a)=\min_{s'}m_{s'}(x)\ge\varepsilon+\sigma\Phi^{-1}(\tau)\).
    Since \(\Phi^{-1}(\tau)>0\) for \(\tau>\tfrac12\) (our deployed \(\tau=0.90\)), the allow region is then contained in \(\{m>\varepsilon\}\) and increases to it as \(\sigma\Phi^{-1}(\tau)\to0\); at \(\tau=\tfrac12\) the containment weakens to \(\{m\ge\varepsilon\}\).
\end{proof}

\begin{proposition}[Lipschitz allow region]
    \label{prop:liparegion}
    Write \(L_{\mathrm{emp}}\le L_{\mathrm{cert}}\) for the gate's empirical continuous Lipschitz constant and \(\delta_{\mathrm{margin}}(x)=m(z,a)-\min_{s'\in\Ndisc(s)}h_\theta(s',x,a)/L_{\mathrm{emp}}\) for its margin underestimation. Then \(\Allow_{\mathrm{Lip}}(z,a)=1\) iff
    \[
        m(z,a)\ >\ \varepsilon+\underbrace{\varepsilon\!\left(\tfrac{L_{\mathrm{cert}}}{L_{\mathrm{emp}}}-1\right)}_{L\text{-slack}}+\underbrace{\delta_{\mathrm{margin}}(x)}_{\text{margin slack}} .
    \]
\end{proposition}

\begin{proof}[Proof of Proposition~\ref{prop:liparegion}]
    By definition \(\Allow_{\mathrm{Lip}}(z,a)=1\) iff \(\min_{s'}h_\theta(s',x,a)>L_{\mathrm{cert}}\varepsilon\), equivalently \(\min_{s'}h_\theta(s',x,a)/L_{\mathrm{emp}}>(L_{\mathrm{cert}}/L_{\mathrm{emp}})\varepsilon\).
    Substituting \(\min_{s'}h_\theta(s',x,a)/L_{\mathrm{emp}}=m(z,a)-\delta_{\mathrm{margin}}(x)\) from the definition of \(\delta_{\mathrm{margin}}\) gives
    \begin{multline*}
        m(z,a)-\delta_{\mathrm{margin}}(x)>\tfrac{L_{\mathrm{cert}}}{L_{\mathrm{emp}}}\varepsilon\\
        \iff
        m(z,a)>\varepsilon\tfrac{L_{\mathrm{cert}}}{L_{\mathrm{emp}}}+\delta_{\mathrm{margin}}(x).
    \end{multline*}
    Writing \(\varepsilon(L_{\mathrm{cert}}/L_{\mathrm{emp}})=\varepsilon+\varepsilon(L_{\mathrm{cert}}/L_{\mathrm{emp}}-1)\) separates the floor \(\varepsilon\) from the \(L\)-slack; since \(L_{\mathrm{emp}}\le L_{\mathrm{cert}}\) the \(L\)-slack is nonnegative.
    When additionally \(\delta_{\mathrm{margin}}(x)\ge 0\) (the gate never overestimates the oracle margin, the fidelity condition measured empirically by \(\texttt{cert\_false\_allow}\)), the allow region is contained in \(\{m>\varepsilon\}\); a gate that overestimates the margin (\(\delta_{\mathrm{margin}}<0\)) can allow points with \(m\le\varepsilon\), which is the empirical fidelity caveat of Table~\ref{tab:headtohead}.
\end{proof}

\begin{definition}[Verified affine policy fragment]
\label{def:fragment}
A policy is in the \emph{fragment} if, for every action \(a\) and every discrete state \(s'\) in the finite declared vocabulary, its restriction to the continuous fields is a finite conjunction of affine constraints,
\[
\Safe((s',x),a)=1 \iff \bigwedge_{j=1}^{m} w_j(s',a)^\top x \le b_j(s',a),
\]
with parameters read syntactically from the policy source (finite categorical branching over \(s'\); no numeric operations beyond affine comparisons). Membership is decided on the policy's abstract syntax tree; a policy outside the fragment is reported \texttt{unsupported} and routed to the verifier ceiling or a gate-certified rung, never silently approximated.
\end{definition}

\begin{proposition}[Exact robust evaluation on the fragment]
\label{prop:fragment}
For a fragment policy, a branch \(s'\in\Ndisc(s)\), and the \(\ell_2\) budget \(\varepsilon\),
\[
\begin{aligned}
&\forall x'\in B_\varepsilon(x):\ \Safe((s',x'),a)=1 \\
&\iff \bigwedge_{j=1}^{m} w_j(s',a)^\top x+\varepsilon\|w_j(s',a)\|_2\le b_j(s',a).
\end{aligned}
\]
Consequently \(\Allow_{\mathrm{exact}}(z,a)\) is computed exactly in \(O(|\Ndisc(s)|\,mk)\) arithmetic operations for \(k\) continuous fields, and \method{}-Exact is policy-certified on the fragment.
\end{proposition}

\begin{proof}
For affine \(g(x')=w^\top x'\), \(\sup_{\|x'-x\|_2\le\varepsilon}w^\top x'=w^\top x+\varepsilon\|w\|_2\), attained at \(x'=x+\varepsilon w/\|w\|_2\) (the support function of the ball; a constraint with \(w=0\) is \(x\)-independent). Hence constraint \(j\) holds on the whole ball iff \(w_j^\top x+\varepsilon\|w_j\|_2\le b_j(s',a)\), and a finite conjunction holds on the ball iff each conjunct does. Enumerating the \(|\Ndisc(s)|\) branches and evaluating \(m\) inner products over \(k\) fields gives the stated cost. Our threshold policies are the \(m{=}1\), coordinate-\(w\) case; the OPA/Rego and JDM suites compile to this fragment, and the sub-microsecond figure of Table~S42 measures exactly this test.
\end{proof}

\section{Extended background}
\label{app:background}

\paragraph{Tool-using agents and harness boundaries.}
An LLM agent runs a think\(\to\)act-via-tool\(\to\)observe loop. The tool has two faces: the executable function (shell command, HTTP request, SQL query, file edit, payment API, monitoring endpoint), which can cause side effects, and the representation shown to the model (name, description, JSON schema), which can be poisoned, especially when tools are discovered dynamically through protocols such as MCP~\cite{anthropic2024mcp,wang2025mcptox}. In Claude Code the model emits structured tool-use blocks while the harness validates, checks permissions, dispatches tools, optionally sandboxes execution, and stores transcripts~\cite{liu2026claudecode}; the safety layer includes rule-based permission checks, hooks, an auto-mode classifier, built-in and MCP tools, sandboxing, and subagents. The right object to certify is thus the runtime decision boundary exposed by the harness rather than the raw LLM.

\paragraph{MCP, tool poisoning, and indirect prompt injection.}
Indirect prompt injection differs from direct injection: malicious instructions enter through third-party content (emails, web pages, retrieved documents, tool outputs, tool descriptions, or agent-to-agent messages)~\cite{greshake2023indirect}, are read by the model, and influence future privileged actions, as quantified by dedicated benchmarks~\cite{zhan2024injecagent,debenedetti2024agentdojo}. MCP-style discovery expands the surface: an external server may publish tool descriptions that become part of the agent context, so tool-metadata poisoning is close to RAG poisoning over tool descriptions, with the added consequence that selected tools have real side effects~\cite{wang2025mcptox,yeon2025toolcert}. We target a different stage: a tool has already returned a typed object, and the question is whether the agent may act on it.

\paragraph{Permission classifiers and their limits.}
Deployed permission systems are mostly pre-execution: Claude Code auto mode classifies proposed tool calls into safe, unsafe, or manual-approval, with a three-tier structure (read-only auto-allowed, in-project edits may bypass the classifier, shell/external actions passed to a transcript classifier), and a stress test reports high false-negative rates on ambiguous DevOps tasks plus a coverage gap where the same state-changing effect is reachable through an unevaluated tool path~\cite{ji2026permissiongate}. Safety thus depends on where the boundary sits as much as on classifier accuracy: a classifier over tool calls does not certify whether a later decision on a returned object is robust.

\section{Fault mechanisms and a worked example}
\label{app:faults}

Throughout, the typed constructor, schema validation, and the integrity and freshness checks constitute the \emph{trusted computing base} (TCB) assumed by the certificate: the guarantee is conditional on this base, and Section~6.5 measures what happens as it weakens (constructor corruption, freshness escape).

The discrete budget \(d=1\) models a single admissible \emph{binding} fault rather than arbitrary string rewriting: \(s'\in\Ndisc(s)\) may be a stale provenance tag, a schema-version mismatch, a confused policy pack, a multiplexed return attributed to the wrong tool, or a TOCTOU race between an environment label and the numerical fields. The continuous budget models bounded distortion of operational quantities that remain subject to range checks, unit normalization, plausibility filters, signed telemetry, or partial integrity checks, so \(\|x'-x\|_2\le\varepsilon\) captures the residual uncertainty left after ordinary validation rather than free choice of values. Section~6.5 injects seven such faults (provenance/policy-pack/TOCTOU swaps, stale cache reads, schema skew, key collisions, numeric jitter, normalization skew) and measures the induced \((d,\varepsilon)\); re-read jitter (\(\texttt{frac\_in\_B}=0.998\)) and normalizer skew (\(0.974\)) are covered fully, while the faults \(\varepsilon\) misses (schema transposition, key collisions) are exactly those a schema/identity validation layer is meant to catch.
The catalogue measures \emph{single} faults; simultaneous faults in one return-assembly window are excluded by an explicit fault-independence assumption (Section~3.2), with the deterministic backend's \(d\le 3\) coverage (Table~\ref{tab:dsweep}) as the fallback if compounding is realistic in a given stack.

\paragraph{Worked example.}
An infrastructure agent may receive an aggregated deployment return whose discrete part \(s\) (\texttt{env}, \texttt{namespace}, \texttt{tool\_source}, \texttt{policy\_pack}) and numerical part \(x\) (\texttt{cpu\_limit}, \texttt{memory}, \texttt{risk\_score}, \texttt{estimated\_cost}) are produced by a \texttt{helm diff}\(\to\)policy scan\(\to\)cost estimate\(\to\)risk score stack. A low-integrity adapter or stale cache can confuse \(\texttt{staging}\leftrightarrow\texttt{prod}\) or \(\texttt{helm\_diff}\leftrightarrow\texttt{external\_report}\) while the numerical fields stay close because they are normalized, range-checked, or signed: exactly a \(d=1\) provenance swap paired with a within-\(\varepsilon\) numerical move. Likewise, a calibrated financial risk score can be numerically plausible yet unsafe under a stricter source (a sanctions screen rather than a credit check), so the same value carries different authorization meaning under the swapped provenance.

\section{Algorithm and experimental setup}
\label{app:algorithm}
\label{app:setup}

\paragraph{Certification algorithm.}
Every rung shares the discrete step: given a typed return \((s,x)\), a candidate action \(a\), and budgets \((d,\varepsilon)\), compute \(\Ndisc(s)=\{s':\Disc(s,s')\le d\}\) exactly and allow only if every branch \(s'\in\Ndisc(s)\) passes its continuous test; the branch test depends on the backend.
\emph{\method{}-Exact (rung 1).} For a fragment policy (Definition~\ref{def:fragment}), test every affine constraint at its ball-worst point, \(w_j^\top x+\varepsilon\|w_j\|_2\le b_j(s',a)\) (Proposition~\ref{prop:fragment}; differential validation against OPA over \(200{,}000\) returns: Table~S48); no learned gate, no sampling, and none of \((\sigma,\tau,M)\).
\emph{\method{}-Lip (rung 2).} Test the deterministic margin \(h_\theta(s',x,a)>L_{\mathrm{cert}}\varepsilon\); no sampling and no confidence parameter (the guarantee is horizon-invariant, Table~\ref{tab:cx2horizon}).
\emph{\method{}-RS (rung 3).} Sample \(\eta_j\sim\mathcal N(0,\sigma^2I)\) (\(j=1,\dots,M\)), estimate the safe probability of \(g_\theta(s',x+\eta_j,a)\), form the Clopper--Pearson lower confidence bound \(\widehat p_{\mathrm{LCB}}(s',x,a)\) at level \(1-\alpha_{\mathrm{FWER}}/|\Ndisc(s)|\), apply the smoothing bound \(\underline p(s',x,a)=\Phi(\Phi^{-1}(\widehat p_{\mathrm{LCB}})-\varepsilon/\sigma)\), and allow iff \(\min_{s'\in\Ndisc(s)}\underline p(s',x,a)\ge\tau\).
RS experiments fix \(\sigma=0.10\), \(\tau=0.90\), \(M=2000\) Monte-Carlo samples per branch (\(M=10^4\) where stated), and \(\alpha_{\mathrm{FWER}}=10^{-3}\), unless a sweep varies them explicitly (Section~6.1); the exact and Lipschitz backends use none of these parameters.
The RS confidence guarantees are \emph{per decision}: a deployment over a horizon of \(T\) decisions must allocate a lifetime budget across decisions (e.g., \(\alpha_t=\alpha_{\mathrm{total}}/T\) or an alpha-spending schedule) or fall back to the deterministic Lipschitz rung, whose guarantee carries no sampling confidence term.

\paragraph{Normalization.}
All continuous fields are normalized before training and certification, and \(\varepsilon,\sigma\) are defined in this normalized space: bounded operational quantities are affinely scaled to \([0,1]\), and heavy-tailed quantities (amounts, latencies) are log-transformed then quantile- or standard-deviation-scaled. Raw units are used only for display and policy descriptions, so \(\varepsilon=0.10\) is a radius in normalized operational-feature space, which is what makes a single \(\varepsilon,\sigma\) meaningful across heterogeneous fields (risk scores, amounts, latencies, counts, error rates).

\paragraph{Computing infrastructure and software.}
All experiments ran on a single workstation running Ubuntu~24.04 LTS (Linux kernel~6.8), with an Intel Xeon Gold~6326 CPU (2.90\,GHz), 495\,GiB of RAM, and one NVIDIA RTX~PRO~6000 Blackwell GPU (96\,GB VRAM; driver 595.58.03, CUDA~13.2).
The analytic core (exact certification, the frozen detectors, and the frozen-protocol scans) is pure Python~3.12 standard library and requires no GPU; the learned backends use \texttt{torch}~2.8.0 with \texttt{orthogonium}~0.0.4 for the 1-Lipschitz gate, plus \texttt{numpy}~2.4.2, \texttt{scipy}~1.17.1, \texttt{scikit-learn}~1.8.0, and \texttt{pandas}~2.3.2.
Policy engines and substrates: OPA~1.17.1, the GoRules ZEN engine (\texttt{zen-engine}~0.53.0), Kyverno on a kind cluster, and the Marble decision API; the LLM proposers are served locally through Ollama, with tags \texttt{qwen2.5:7b-instruct}, \texttt{qwen2.5-coder:7b}, \texttt{qwen2.5:32b}, and \texttt{qwen3.6:latest} (the last a 36B mixture-of-experts, Q4\_K\_M quantization, \(\approx23\)\,GB).
All timing figures reported in Section~6 were measured on this machine.
Code, exact run commands, and per-experiment capability tags will be made publicly available.

\section{Additional result tables}
\label{app:tables}

These tables support Sections~6.2--6.5; the main text keeps one key table per claim and relocates the secondary ones here.
Zero cells across these tables carry explicit \(n/N\) and Wilson-95\% upper bounds, audited in Table~\ref{tab:wilson}.

\FloatBarrier
\subsection{Core certificate evaluation}
Primary real-data certificate evaluation, backend study, and controlled scaling diagnostics.

\medskip\noindent\textbf{Question.} Is the joint certificate sound and non-vacuous on real transaction marginals?
\textbf{Setting.} Real IEEE-CIS marginals, \(10{,}000\) records per sampling mode (boundary-balanced, natural, \(C\)-targeted).
\textbf{Result.} Sound in every mode (\(\Callow=\Uallow=\texttt{cert\_false\_allow}=0\)) and non-vacuous on \(R\); no fraud among certified allows against a \(2.7\)--\(3.8\%\) base rate (Table~\ref{tab:ieeecis}).

\begin{table}[h]
    \centering
    \small
    \caption{Joint certificate on real IEEE-CIS marginals (Section~6.4). Sound (\(\Callow=\Uallow=\texttt{cert\_false\_allow}=0\)) and non-vacuous on \(R\) across sampling modes; fraud rate all/certified-allowed/blocked: \(3.8/0.0/4.1\%\) (boundary-balanced), \(2.8/0.0/3.3\%\) (natural), \(2.7/0.0/2.8\%\) (\(C\)-targeted); \(10{,}000\) records per mode.}
    \label{tab:ieeecis}
    \resizebox{\columnwidth}{!}{%
        \begin{tabular}{lccccc}
            \toprule
            sampling          & \(|C|\) & \(\texttt{clean\_acc}\) & \(\Rallow\) & \(\texttt{naive\_C}\) & fraud (allowed) \\
            \midrule
            boundary-balanced & 2000    & 0.9985                  & 0.225       & 1.0                   & 0.00            \\
            natural           & 636     & 0.9985                  & 0.200       & 1.0                   & 0.00            \\
            \(C\)-targeted    & 6000    & 0.9990                  & 0.238       & 1.0                   & 0.00            \\
            \bottomrule
        \end{tabular}}
\end{table}

\medskip\noindent\textbf{Question.} Do the phenomenon and the certificate transfer to a second real dataset outside finance?
\textbf{Setting.} NAB CloudWatch CPU telemetry (\(\approx58\)k rows) under a staging/production monitoring policy (\(\delta=0.08\), real q\(0.70\) base threshold); \(n=6000\), \(3\) seeds.
\textbf{Result.} Natural \(C=10.91\pm2.01\%\) with \(600\) audited witnesses and \(\texttt{naive\_C\_falseallow}=1.0\); all three backends are sound, with the numeric-block scaling held-out validated (Table~\ref{tab:nab}).

\begin{table}[h]
    \centering
    \small
    \caption{Second real dataset (Section~6.2): NAB AWS CloudWatch CPU telemetry~\cite{lavin2015nab} (\(\approx 58\)k rows) under a constructed staging/production monitoring policy (\(\theta(s)\) gap \(\delta=0.08\), \(\theta_{\mathrm{base}}\) the real CPU q\(0.70\) quantile); \(n=6000\), \(3\) seeds, \(\varepsilon=0.10\). Natural \(C=10.91\pm2.01\%\); \(600\) audited joint-gap witnesses with zero violations; \(\texttt{naive\_C\_falseallow}=1.0\). Numeric-block feature scaling (\(\texttt{fscale}=4\), certified constant \(\texttt{fscale}\cdot L_{\mathrm{cert}}\)) resolves the Lipschitz underfit soundly. The NAB anomaly label is a diagnostic only, never the certification label.
The \(\texttt{fscale}\) choice is held-out validated: a stratified selection/eval split selects \(\texttt{fscale}^{*}=6.0\) on every seed with eval \(\texttt{cert\_false\_allow}=0\) and \(\Rallow=1.0\) (no soundness cliff on this grid; sound at \(d=2\) as well).}
    \label{tab:nab}
    \resizebox{\columnwidth}{!}{%
        \begin{tabular}{lccc}
            \toprule
            backend                                    & clean acc.\       & \(\texttt{cert\_false\_allow}\) & \(\Rallow\) \\
            \midrule
            Lipschitz (primary, \(\texttt{fscale}=4\)) & \(0.831\pm0.012\) & \(0\)                           & \(1.00\)    \\
            randomized smoothing (MLP)                 & \(0.999\)         & \(0\)                           & \(1.00\)    \\
            exact predicate (ceiling)                  & \(1.000\)         & \(0\)                           & \(1.00\)    \\
            \bottomrule
        \end{tabular}}
\end{table}

\medskip\noindent\textbf{Question.} Do the results survive when authorization labels come from a real policy engine instead of the analytic oracle?
\textbf{Setting.} \texttt{opa eval} over authored provenance-conditioned Rego (\(5\) seeds), then an operating-point sweep of \(24\) cells (\(3\) seeds \(\times\,\varepsilon\times\tau\times\) backend \(\times\) domain, \(10{,}800\) records, per-record FWER correction).
\textbf{Result.} Joint-gap witnesses arise spontaneously at \(\approx10\)--\(12\%\), the certified gate is sound with zero seed variance, and every sweep cell is sound (Tables~\ref{tab:opa}--\ref{tab:opasweep}).

\begin{table}[h]
    \centering
    \small
    \caption{Policy-as-code grounding (Section~6.4). Authorization labels are produced by a real Open Policy Agent engine over authored provenance-conditioned Rego (mean \(\pm\) std over 5 seeds, \(\varepsilon=\sigma=0.10\), \(\tau=0.90\)). Joint-gap witnesses (category \(C\)) arise spontaneously at \(\approx 10\)--\(12\%\); the certified gate is sound with zero seed variance.}
    \label{tab:opa}
    \resizebox{\columnwidth}{!}{%
        \begin{tabular}{lccccc}
            \toprule
            domain  & \(C\) prevalence  & \(\Rallow\)       & \(\texttt{cert\_false\_allow}\) & \(\Callow\)       & \(\Uallow\)       \\
            \midrule
            finance & \(0.112\pm0.007\) & \(0.067\pm0.010\) & \(0.000\pm0.000\)               & \(0.000\pm0.000\) & \(0.000\pm0.000\) \\
            SRE     & \(0.107\pm0.016\) & \(0.073\pm0.018\) & \(0.000\pm0.000\)               & \(0.000\pm0.000\) & \(0.000\pm0.000\) \\
            ops     & \(0.122\pm0.011\) & \(0.065\pm0.017\) & \(0.000\pm0.000\)               & \(0.000\pm0.000\) & \(0.000\pm0.000\) \\
            \bottomrule
        \end{tabular}}
\end{table}

\begin{table}[h]
    \centering
    \small
    \caption{OPA operating-point sweep (Section~6.4). Real \texttt{opa eval} labels with a per-record family-wise correction \(\alpha_{\mathrm{branch}}=\alpha_{\mathrm{FWER}}/|\Ndisc(s)|\); \(3\) seeds \(\times\,\varepsilon\in\{0.05,0.10\}\times\tau\in\{0.90,0.95\}\times\{\text{smoothing, Lipschitz}\}\times 3\) domains (\(10{,}800\) records). All \(24\) cells are sound (\(\Callow=\Uallow=\texttt{cert\_false\_allow}=0\)) while the uncertified gate's attack-false-allow spans \(0.30\)--\(0.47\); only utility moves. A separate \(60\)-point \(\varepsilon\times\tau\) grid over the three synthetic domains (\(\varepsilon\in\{0.03,\dots,0.20\}\times\tau\in\{0.80,\dots,0.95\}\)) and a \(\sigma\in[0.05,0.20]\) ablation are likewise sound at every point, with finance \(\Rallow\) falling \(0.68\to0.49\to0\) as \((\varepsilon,\tau)\) tightens. The three cells where smoothing goes vacuous at \((\varepsilon,\tau)=(0.10,0.95)\) are a utility failure rather than a soundness one; the deterministic Lipschitz backend recovers \(\Rallow=0.36\)--\(0.51\) there.}
    \label{tab:opasweep}
    \resizebox{0.95\columnwidth}{!}{%
        \begin{tabular}{cclccl}
            \toprule
            \(\varepsilon\) & \(\tau\) & backend   & sound & non-vacuous & note                       \\
            \midrule
            0.05            & 0.90     & smoothing & 3/3   & 3/3         &                            \\
            0.05            & 0.95     & smoothing & 3/3   & 3/3         &                            \\
            0.10            & 0.90     & smoothing & 3/3   & 3/3         & Table~\ref{tab:opa}        \\
            0.10            & 0.95     & smoothing & 3/3   & 0/3         & vacuous (\(\Rallow=0\))    \\
            0.05            & 0.90     & Lipschitz & 3/3   & 3/3         &                            \\
            0.05            & 0.95     & Lipschitz & 3/3   & 3/3         &                            \\
            0.10            & 0.90     & Lipschitz & 3/3   & 3/3         & Table~\ref{tab:lip}        \\
            0.10            & 0.95     & Lipschitz & 3/3   & 3/3         & \(\Rallow=0.36\)--\(0.51\) \\
            \bottomrule
        \end{tabular}}
\end{table}

\medskip\noindent\textbf{Question.} How do the certified backends compare on one fixed task, and how much utility does the learned margin leave unclaimed?
\textbf{Setting.} The same trained OPA-track gate across \(75\) multi-seed cells, plus a complete big-\(M\) MILP verification of the identical gate as ceiling (\(3\) seeds, \(\varepsilon=0.10\)).
\textbf{Result.} Every cell is sound; smoothing a bounded-logit \texttt{LipGate} is structurally near-vacuous; the complete ceiling certifies the gate itself, locating the utility deficit in the learned margin rather than in the certificate (Tables~\ref{tab:lip}--\ref{tab:cv}).

\begin{table}[h]
    \centering
    \small
    \caption{Backend comparison on the OPA track (Experiment LIP; Section~6.4). All \(75\) multi-seed cells are sound (\(\Callow=\Uallow=\texttt{cert\_false\_allow}=0\)); \(\varepsilon=0.10\) deterministic entries are mean over five seeds (finance/SRE/ops), others ranges across domains. Smoothing a bounded-logit \texttt{LipGate} is structurally near-vacuous (\(\Rallow\approx 0\)) and omitted as a backend--architecture mismatch. At the strict point the finite-MC tax is \(\approx 0\) (the deterministic backend is within seed variance of \(M{=}10^4\) smoothing on two of three domains) and a certified local gradient bound (\(\approx 0.98\approx L_{\mathrm{cert}}=1\)) leaves no removable \(L\)-slack: the residual deficit to exact utility is learned-margin deficiency. Margin geometry: allowed-\(R\) points sit at mean oracle margin \(0.378\), refused-\(R\) at \(0.189\), straddling \(\varepsilon+\sigma\Phi^{-1}(\tau)=0.228\) (correlation \(0.86\)): the smoothed allow region converges to the abstention floor from above (Proposition~5).}
    \label{tab:lip}
    \resizebox{\columnwidth}{!}{%
        \begin{tabular}{llccc}
            \toprule
            model                                   & backend                                  & \(\Rallow\,(\varepsilon{=}0.10)\) & \(\Rallow\,(\varepsilon{=}0.03)\) & ms  \\
            \midrule
            MLP                                     & \(\Allow_{\mathrm{RS}}\) (\(M{=}2000\))  & \(0.13\)--\(0.19\)                & \(0.40\)--\(0.47\)                & --- \\
            MLP                                     & \(\Allow_{\mathrm{RS}}\) (\(M{=}10000\)) & \(0.23\)--\(0.32\)                & \(0.43\)--\(0.52\)                & 46  \\
            LipGate                                 & \(\Allow_{\mathrm{Lip}}\)                & \(0.35/0.37/0.27\)                & \(0.41\)--\(0.54\)                & 7.8 \\
            \midrule
            \multicolumn{2}{l}{exact (rung 1, OPA)} & \(1.00\)                                 & \(1.00\)                          & \(<0.001\)                              \\
            \bottomrule
        \end{tabular}}
\end{table}

\begin{table}[h]
    \centering
    \small
    \caption{Complete-verification ceiling (Section~5.1): the same trained OPA-track gate certified per branch by big-\(M\) MILP over \(\|\delta\|_2\le\varepsilon\) (circumscribing outer polytope: complete for the polytope, sound and possibly conservative for the ball), vs the Lipschitz and smoothing backends; \(3\) seeds, \(\varepsilon=0.10\), \(\Rallow\) as finance/SRE/ops. CV certifies the gate's entire robust-safe set; its \(0.0085\) oracle false-allow on ops (the multivariate rule) is base-gate fidelity passed through exactly: completeness w.r.t.\ the gate is not soundness w.r.t.\ the policy. MILP branch; \(\alpha,\beta\)-CROWN unavailable in this environment.}
    \label{tab:cv}
    \resizebox{\columnwidth}{!}{%
        \begin{tabular}{lccc}
            \toprule
            backend                      & \(\Rallow\) (fin/SRE/ops) & \(\texttt{cert\_false\_allow}\) & ms/record        \\
            \midrule
            Lipschitz (primary)          & \(0.32/0.35/0.37\)        & \(0\)                           & \(7\)--\(17\)    \\
            smoothing (\(M{=}10^4\))     & \(0.32/0.33/0.34\)        & \(0\)                           & \(61\)--\(132\)  \\
            complete verification (MILP) & \(1.00/1.00/1.00\)        & \(0/0/0.0085\)                  & \(\approx 1700\) \\
            \bottomrule
        \end{tabular}}
\end{table}

\medskip\noindent\textbf{Question.} What does certification cost when engine-exact ground truth is available?
\textbf{Setting.} \(R_{\mathrm{OPA}}\) = the engine-exact robust-safe set; boundary-balanced evaluation, \(\varepsilon=0.10\), \(d=1\), \(3\) seeds, \(800\) records/domain.
\textbf{Result.} Point systems pay a boundary price; the gate-certified Lipschitz backend holds policy false-allow at \(0\) (Wilson-95\% upper \(\le0.022\)) while recovering a substantial fraction of \(R_{\mathrm{OPA}}\) (Table~\ref{tab:cx1fidelity}).

\begin{table}[h]
\centering
\small
\caption{Gate--policy fidelity benchmark under known ground truth (EXP-CX1; Section~6.4). \(R_{\mathrm{OPA}}\) = engine-exact robust-safe set; boundary-balanced evaluation, \(\varepsilon=0.10\), \(d=1\), 3 seeds, \(n_{\mathrm{eval}}=800\)/domain. Point systems pay a boundary price; the gate-certified Lipschitz backend holds policy false-allow \(0\) (Wilson 95\% upper \(\le0.022\)) while recovering \(\approx0.65\) of \(R_{\mathrm{OPA}}\) (the quantified learning tax vs the policy-certified exact backend). RS at \(n_{\mathrm{mc}}=800\) is vacuous (documented budget choice; Lipschitz is the primary).}
\label{tab:cx1fidelity}
\resizebox{\columnwidth}{!}{%
\begin{tabular}{lcccc}
\toprule
system & policy false-allow & robust-safe coverage & Jaccard vs \(R_{\mathrm{OPA}}\) & point acc.\ \\
\midrule
OPA at the point        & \(0.30\)  & ---            & ---            & \(1.00\) \\
point MLP (uncertified) & \(0.23\)  & ---            & ---            & \(\approx0.91\) \\
\method{}-Exact         & \(0.000\) & \(1.0\)        & \(1.0\)        & \(1.00\) \\
\method{}-Lip           & \(0.000\) (\(\le0.022\)) & \(\approx0.65\) & \(0.65\) & \(\approx0.91\) \\
\method{}-RS (\(n_{\mathrm{mc}}{=}800\)) & \(0.000\) & \(\approx0.016\) & \(\approx0.016\) & --- \\
\bottomrule
\end{tabular}}
\end{table}

\medskip\noindent\textbf{Question.} Is the joint gap an artifact of one scale?
\textbf{Setting.} Three controlled sweeps \(\times\,8000\) synthetic records, \(d=1\), action-partitioned tool vocabularies.
\textbf{Result.} Joint-gap prevalence stays at \(7.8\)--\(9.7\%\), the certificate is sound at every point, and the uncertified learned gate's attack false-allow spans \(0.90\)--\(0.98\) --- the \(90\)--\(98\%\) anchor (Table~\ref{tab:scaling}).

\begin{table}[h]
\centering
\small
\caption{Scaling study (Experiment A; anchors the uncertified \(90\)--\(98\%\) attack success). Three sweeps \(\times\,8000\) records, \(d=1\), synthetic typed tools with action-partitioned vocabularies. Joint-gap prevalence stays \(7.8\)--\(9.7\%\), the certificate is sound at every point, and the uncertified learned gate's \(\texttt{attack\_false\_allow}\) spans \(0.90\)--\(0.98\) while naive composition is \(1.0\).}
\label{tab:scaling}
\resizebox{\columnwidth}{!}{%
\begin{tabular}{lccccc}
\toprule
sweep & \(C\) prevalence & clean acc & \(\texttt{cert\_false\_allow}\) & \(\Rallow\) & uncert.\ \(\texttt{attack\_FA}\) \\
\midrule
tool vocabulary \(K\in\{4,\dots,32\}\)   & \(7.8\)--\(9.7\%\) & \(0.974\)--\(0.997\) & \(0\) & \(0.35\)--\(0.625\) & \(0.90\)--\(0.98\) \\
numeric dim.\ \(k\in\{2,\dots,50\}\)     & \(7.8\)--\(9.7\%\) & \(0.974\)--\(0.997\) & \(0\) & \(0.35\)--\(0.625\) & \(0.90\)--\(0.98\) \\
categorical \(|X_1|\in\{2,4,8\}\)        & \(7.8\)--\(9.7\%\) & \(0.974\)--\(0.997\) & \(0\) & \(0.35\)--\(0.625\) & \(0.90\)--\(0.98\) \\
\bottomrule
\end{tabular}}
\end{table}

\medskip\noindent\textbf{Question.} Where do the PSD2/AML anchors and the additional realistic-schema and DevOps evaluations come from?
\textbf{Setting.} PSD2/AML source-locked regulatory thresholds (natural sampling), realistic-schema threshold-gap policies (\(50{,}000\) records/domain), and benchmark-grounded DevOps fixtures.
\textbf{Result.} Each setting's provenance, sampling scheme, and oracle are pinned per row (Table~\ref{tab:settinganchors}).

\begin{table}[h]
\centering
\small
\caption{Per-setting anchors for PSD2/AML and the additional realistic-schema and DevOps evaluations (Experiments REG, B, G). PSD2/AML: source-locked regulatory thresholds (PSD2 Art.~11/16/18 exemption values, US CTR), authored executably, natural sampling. Realistic schemas: \(50{,}000\) records/domain, synthetic threshold-gap policies. DevOps: \(9{,}600\) benchmark-grounded fixture records under the \texttt{hybrid\_policy} oracle (boundary-constructed, hence the higher witness share); its \(\Rallow\approx0.11\) is a conservative lower bound (\(\approx27\%\) analytically certifiable).}
\label{tab:settinganchors}
\resizebox{\columnwidth}{!}{%
\begin{tabular}{lcccc}
\toprule
setting & joint-gap prevalence & \(\texttt{cert\_false\_allow}\) & \(\Rallow\) (\(\varepsilon{=}0.10\)) & \(\Rallow\) (\(\varepsilon{=}0.03\)) \\
\midrule
PSD2/AML (REG)      & \(6.5\)--\(9.8\%\) (natural)  & \(0\) & \(0.20\)--\(0.34\) & \(0.44\)--\(0.56\) \\
realistic schemas (B) & \(8.0\)--\(8.6\%\)          & \(0\) & \(0.24\)--\(0.38\) & --- \\
DevOps fixture (G)  & \(1661/9600\) (constructed)   & \(0\) & \(\approx0.11\)    & --- \\
\bottomrule
\end{tabular}}
\end{table}

\medskip\noindent\textbf{Question.} Does an adaptive attacker who knows the gate break it?
\textbf{Setting.} Adaptive in-budget attack (TM2) against a locally-correct learned gate and against the certified gate.
\textbf{Result.} The learned gate admits true-unsafe points; the certified gate removes them all (\(0.000\)) while staying non-vacuous on \(R\) (Table~\ref{tab:tm2}).

\begin{table}[h]
    \centering
    \scriptsize
    \setlength{\tabcolsep}{4pt}
    \caption{Adaptive attack on the gate (Section~6.4). A locally-correct learned gate admits true-unsafe points; the certified gate removes them (\(0.000\)) while staying non-vacuous on \(R\).}
    \label{tab:tm2}
    \resizebox{\columnwidth}{!}{%
        \begin{tabular}{lcccccc}
            \toprule
            domain             & learned adv.\ FA & certified adv.\ FA & \(\Callow\) & \(\Uallow\) & \(\Rallow\) & \(\texttt{cert\_false\_allow}\) \\
            \midrule
            finance compliance & 0.017            & 0.000              & 0.000       & 0.000       & 0.495       & 0.000                           \\
            SRE monitoring     & 0.150            & 0.000              & 0.000       & 0.000       & 0.280       & 0.000                           \\
            ops security       & 0.043            & 0.000              & 0.000       & 0.000       & 0.230       & 0.000                           \\
            \bottomrule
        \end{tabular}}
\end{table}

\medskip\noindent\textbf{Question.} Does soundness survive held-out shift?
\textbf{Setting.} Held-out policies, schemas, and thresholds, with clean accuracy and certified false allow reported for each condition.
\textbf{Result.} Certified and certified-adversarial false allow are \(0\) in every condition; under threshold shift the learned gate degrades sharply, and both effects are reported rather than averaged away (Table~\ref{tab:generalization}).

\begin{table}[h]
    \centering
    \scriptsize
    \setlength{\tabcolsep}{4pt}
    \caption{Held-out generalization (Section~6.4). Soundness holds across all conditions (\(\texttt{cert\_false\_allow}=\) cert.\ adv.\ FA \(=0\)); under threshold shift the learned gate degrades sharply and utility moves, and both are reported. The headline numbers are reported inline (Section~6.4); this table adds clean accuracy and \(\texttt{cert\_false\_allow}\).}
    \label{tab:generalization}
    \resizebox{\columnwidth}{!}{%
        \begin{tabular}{lccccccc}
            \toprule
            condition          & \(\texttt{clean\_acc}\) & \(\Rallow\) & \(\Callow\) & \(\Uallow\) & \(\texttt{cert\_false\_allow}\) & learned adv.\ FA & cert.\ adv.\ FA \\
            \midrule
            in-distribution    & 1.00                    & 0.433       & 0.00        & 0.00        & 0.00                            & 0.017            & 0.00            \\
            held-out-threshold & 0.95                    & 0.567       & 0.00        & 0.00        & 0.00                            & 0.475            & 0.00            \\
            held-out-tool      & 1.00                    & 0.233       & 0.00        & 0.00        & 0.00                            & 0.092            & 0.00            \\
            \bottomrule
        \end{tabular}}
\end{table}

\medskip\noindent\textbf{Question.} What does each of the two channels contribute on real features?
\textbf{Setting.} IEEE-CIS features with display-channel attacks (fixed set of \(8\), adaptive set of \(13\); \(n=120\)/cell) and the \(\Ball\) certificate over the typed return (category-balanced \(n=600\)).
\textbf{Result.} Point gates are locally correct yet keep \(\Callow=1.0\); only the joint certificate closes the joint gap (Table~\ref{tab:realdata}).

\begin{table*}[t]
    \centering
    \small
    \caption{Real-data two-channel run on IEEE-CIS features (Section~6.4): display-channel attacks (fixed set of 8 and adaptive set of 13) against the untrusted text, and the \(\Ball\) certificate over the typed return; constructed provenance-threshold policy, \(n=120\)/cell (display) and category-balanced \(n=600\) (certificate). Point gates are locally correct yet keep \(\Callow=1.0\); only the certified gate is sound, at \(\Rallow=0.20\).}
    \label{tab:realdata}
        \begin{tabular}{lcccc}
            \toprule
            gate      & display \(\texttt{UnsafeExec}_U\) (best-of-\(K\)) & \(\Callow\) & \(\texttt{cert\_false\_allow}\) & \(\Rallow\) \\
            \midrule
            none      & \(1.00\)                                          & \(1.00\)    & \(0.80\)                        & \(1.00\)    \\
            rule      & \(0.00\)                                          & \(1.00\)    & \(0.75\)                        & \(1.00\)    \\
            learned   & \(0.00\)                                          & \(1.00\)    & \(0.75\)                        & \(1.00\)    \\
            certified & \(0.00\)                                          & \(0.00\)    & \(0.00\)                        & \(0.20\)    \\
            \bottomrule
        \end{tabular}
\end{table*}

\medskip\noindent\textbf{Question.} What remains when there is no executable policy at all?
\textbf{Setting.} Policy = the real IEEE-CIS \(\texttt{isFraud}\) label (gate AUC \(\approx0.72\)); exact and marginal certificates are undefined in this regime.
\textbf{Result.} \method{}-RS reduces empirical fraud exposure; the formal certificate remains a certificate of the smoothed gate, with the gate--policy fidelity gap made explicit (Table~\ref{tab:implicit}).

\begin{table}[h]
    \centering
    \small
    \caption{Implicit-policy regime as a gate-fidelity stress test (Section~6.4): policy = the real IEEE-CIS \(\texttt{isFraud}\) label (gate AUC \(\approx 0.72\); no executable predicate, so exact and marginal certificates are undefined). \method{}-RS reduces empirical fraud exposure in this low-fidelity regime, and the formal certificate remains a certificate of the smoothed gate: the gate--policy fidelity assumption stays. The deterministic backend keeps a stable certified fraction at low sample size where smoothing collapses; ground truth is the held-out fraud label (empirical, never a predicate-soundness theorem).}
    \label{tab:implicit}
    \resizebox{\columnwidth}{!}{%
        \begin{tabular}{lccc}
            \toprule
                                                     & \method{}-Lip & \method{}-RS       & point gate (matched) \\
            \midrule
            certified fraction of safe txns          & \(0.72\)      & \(0.40\)           & ---                  \\
            fraud false-allow under \(\Ball\) attack & \(0.41\)      & \(\mathbf{0.044}\) & \(0.47\)--\(0.49\)   \\
            \bottomrule
        \end{tabular}}
\end{table}

\FloatBarrier
\subsection{End-to-end systems and comparators}
Real LLM proposers, deployed-defense comparators, and committed side effects on real engines.

\paragraph{Engine-verified joint-gap witnesses (Section~6.2).}
The provenance-conditioned idiom \(\texttt{risk\_score} < \theta(\text{provenance})\) is authored executably with the real generation constants (\(\theta=0.488808\), \(\delta=0.08\)), once as Rego and once as a JDM decision table, and every point of the probe set (clean, the \(d{=}1\) loose\(\leftrightarrow\)strict provenance swap, the \(+\varepsilon\) numeric move, and the joint swap\(+\varepsilon\)) is labelled by the engine itself over real boundary-balanced IEEE-CIS transactions.
Both engines, OPA~1.17.1 and GoRules ZEN~0.53 (a Rust decision engine used in fintech), label \(800\) category-\(C\) witnesses over \(4000\) transactions each and reproduce the analytic taxonomy exactly (engine--analytic agreement \(1.000\), \(C\)-set Jaccard \(1.000\)).
The \(20\%\) rate is the engineered balanced-set rate; natural prevalence is \(\approx 6.4\%\) (Table~\ref{tab:ieeecis}).
Neither engine is a purpose-built fraud/AML engine: the claim is engine independence of the joint gap, and deployed-rule discovery remains a separate question.
A purpose-built AML engine confirms the same labels (Table~\ref{tab:marble}), and a second independent adapter reproduces both the taxonomy and the marginal-vs-joint verdicts through real Kyverno admission (Table~\ref{tab:a7}).

\medskip\noindent\textbf{Question.} Is the threat model-dependent while the defense is model-independent?
\textbf{Setting.} Four real proposers (three Qwen2.5 variants and the 36B-MoE \texttt{qwen3.6}) proposing actions on IEEE-CIS, mock control included.
\textbf{Result.} Undefended leak varies substantially across models and domains; the measured certified false-allow rate is \(0.000\) for every model and domain, and gate-level outcomes are identical across proposers, which pre-gate approve \(100\%\) of category-\(C\) transactions (Tables~\ref{tab:q7model}--\ref{tab:ieeecisagent}).

\begin{table}[h]
    \centering
    \small
    \caption{Model-dependent threat vs.\ model-independent defense (Section~6.4). Undefended leak (oracle \(\texttt{cert\_false\_allow}\) under the joint-gap attack) varies substantially across models and domains; the certified gate is \(0.000\) for every model and domain. Parse rate \(1.000\) for all models; certified \(R_{\mathrm{exec}}\) (finance) \(0.342\) (32B) and \(0.350\) (\texttt{qwen3.6}).}
    \label{tab:q7model}
    \resizebox{\columnwidth}{!}{%
        \begin{tabular}{llcc}
            \toprule
            model                        & domain  & no gate (joint-gap) & certified \\
            \midrule
            \texttt{qwen2.5-coder:7b}    & finance & 0.385                   & 0.000     \\
            \texttt{qwen2.5-coder:7b}    & SRE     & 0.639                   & 0.000     \\
            \texttt{qwen2.5:7b-instruct} & finance & 0.087                   & 0.000     \\
            \texttt{qwen2.5:7b-instruct} & SRE     & 0.118                   & 0.000     \\
            \texttt{qwen2.5:32b}         & finance & 0.158                   & 0.000     \\
            \texttt{qwen2.5:32b}         & SRE     & 0.448                   & 0.000     \\
            \texttt{qwen3.6} (36B MoE)   & finance & 0.474                   & 0.000     \\
            \texttt{qwen3.6} (36B MoE)   & SRE     & 0.436                   & 0.000     \\
            \bottomrule
        \end{tabular}}
\end{table}

\begin{table}[h]
    \centering
    \small
    \caption{Real-LLM action proposal on IEEE-CIS (Section~6.4), unsafe execution rate. Gate-level outcomes are identical across mock, 7B, and 32B; the model-dependent threat is in Table~\ref{tab:q7model}. Pre-gate, both 7B and 32B approve \(100\%\) of category-\(C\) transactions (paired display-only flip rate \(0.73\)--\(0.93\)) and comply with \(100\%\) of display-channel injections.}
    \label{tab:ieeecisagent}
    \resizebox{\columnwidth}{!}{%
        \begin{tabular}{lcccc}
            \toprule
            attack     & no gate & learned & certified & \(\texttt{cert\_false\_allow}\) \\
            \midrule
            c\_witness & 1.00    & 1.00    & 0.00      & 0.00                            \\
            injection  & 1.00    & 0.00    & 0.00      & 0.00                            \\
            \bottomrule
        \end{tabular}}
\end{table}

\medskip\noindent\textbf{Question.} Can untrusted display text steer unsafe execution through a typed gate?
\textbf{Setting.} Display-channel injections (TM1) against every system, including a robust prompt and a third model (\texttt{qwen3.6}).
\textbf{Result.} Display text moves the LLM proposal for every system, but typed gates make unsafe execution invariant to it (spread \(0.00\)); a robust prompt does not fix the undefended systems (Table~\ref{tab:tm1}).

\begin{table}[h]
    \centering
    \scriptsize
    \setlength{\tabcolsep}{4pt}
    \caption{Non-instructability (TM1; Section~6.4). Display text \(m\) moves the LLM proposal for every system, but typed gates make unsafe execution invariant to \(m\) (spread \(0.00\)). A robust prompt does not fix it. Reproduced with a third model (EXP-C3, \texttt{qwen3.6} 36B MoE): undefended \(\texttt{unsafe\_exec}\) \(0.85\)--\(1.00\), \(\texttt{llm\_flip}\) \(0.80\)--\(0.83\), typed gates invariant (\(\texttt{gate\_flip}=0\)); best-of-\(K\) pushes the strongest prompt baseline to \(0.100\) while the typed gate holds \(0\).}
    \label{tab:tm1}
    \resizebox{\columnwidth}{!}{%
        \begin{tabular}{lcccc}
            \toprule
            system                     & gate input & \(\texttt{priv\_prop}\) (benign\(\to\)inj) & \(\texttt{unsafe\_exec}\) (benign\(\to\)inj) & spread over \(m\) \\
            \midrule
            LLM-only (none)            & \((z,m)\)  & \(0.85\to 1.00\)                           & \(0.187\to 0.333\)                           & \(0.21\)          \\
            LLM + robust prompt        & \((z,m)\)  & \(0.85\to 1.00\)                           & \(0.173\to 0.333\)                           & \(0.21\)          \\
            LLM + learned typed gate   & \(z\) only & \(0.85\to 1.00\)                           & \(0.000\to 0.000\)                           & \(0.00\)          \\
            LLM + certified typed gate & \(z\) only & \(0.85\to 1.00\)                           & \(0.000\to 0.000\)                           & \(0.00\)          \\
            \bottomrule
        \end{tabular}}
\end{table}

\medskip\noindent\textbf{Question.} Can an LLM judge substitute for the certificate?
\textbf{Setting.} Judges evaluate the serialized observed pair \((z,a)\) on the witness set; a budget-aware judge additionally receives the exact \(\Ball\) and is asked for the worst case (\(3\) seeds, benign paraphrases).
\textbf{Result.} The budget-aware judge's verdict spans the full \([0,1]\) across three benign paraphrases of the same instruction: not a sound gate (Table~\ref{tab:judge}).

\begin{table}[h]
    \centering
    \small
    \caption{LLM-judge baselines on the witness set \(W\) (Section~6.2). Judges evaluate the serialized observed pair \((z,a)\); the certified gate evaluates \(\Ball(z)\). The budget-aware judge receives the exact \(\Ball\) in its prompt and is asked for the worst case; its verdict spans the full \([0,1]\) across three benign paraphrases of the same instruction (\(3\) seeds each, decode temperature \(0.7\)). Judge model \texttt{qwen2.5:32b}~\cite{qwen2024qwen25}; guard model \texttt{llama-guard3:1b}~\cite{inan2023llamaguard}. \(\texttt{cert\_false\_allow}\) is the empirical fraction of allows that are joint-unsafe. Stronger-judge leg (EXP-C3): \texttt{qwen3.6} (36B MoE, chain-of-thought off) at decode temperature \(0.7\) and \(0\); at temperature \(0\) the three decode seeds are identical per phrasing (one phrasing \(0.617\), the others \(0.000\)), so the residual spread is phrasing sensitivity, not sampling noise. A tool-augmented judge that reproduced the enumeration would concede the locus argument.}
    \label{tab:judge}
    \resizebox{\columnwidth}{!}{%
        \begin{tabular}{llcc}
            \toprule
            gate                       & evaluates    & allow-on-\(W\)  & \(\texttt{cert\_false\_allow}\) \\
            \midrule
            LLM point judge            & point        & \(0.92\)        & \(0.55\)                        \\
            budget-aware LLM judge     & point        & \(0.44\pm0.43\) & \(0.27\pm0.26\)                 \\
            stronger point judge (36B MoE)         & point & \(0.258\)     & ---                             \\
            stronger budget-aware judge (36B MoE)  & point & \(0.14\pm0.22\) (max \(0.61\)) & up to \(0.31\) \\
            guard model                & point        & \(0.50\)        & \(0.30\)                        \\
            certified gate (reference) & neighborhood & \(0.000\)       & \(0.000\)                       \\
            \bottomrule
        \end{tabular}}
\end{table}

\medskip\noindent\textbf{Question.} Do deployed defense hooks admit the joint-gap witness?
\textbf{Setting.} Runtime enforcement of the observed action (AgentSpec, VeriGuard), the capability defense CaMeL, and pre-execution screening, all on the same witness set (Experiment P5).
\textbf{Result.} A witness is safe at the observed point, so every point-evaluating defense admits it; only the joint certificate blocks it, and CaMeL and the certificate are orthogonal (Table~\ref{tab:comparators}).

\begin{table*}[t]
\centering
\small
\caption{Deployed-defense comparators (Experiment P5; Sections~2 and~6.3). \emph{Summary:} a joint-gap witness is safe at the observed point, so every point-evaluating defense admits it; runtime enforcement of the observed action (AgentSpec~\cite{wang2026agentspec}, VeriGuard~\cite{miculicich2025veriguard}) and the capability defense CaMeL~\cite{debenedetti2025camel} both allow it, and only the neighborhood certificate blocks it. \(^{*}\)CaMeL is orthogonal rather than dominated: on the real finance schema it blocks source-tainted values that \method{} allows and admits the in-scope-provenance witness that \method{} blocks, so the two are composable and neither subsumes the other (\emph{orthogonality} rows). \emph{Classifier:} a maximally generous pre-execution transcript classifier (full transcript, point accuracy \(1.00\)) used as an authorize-then-realize gate still commits unsafe effects at the no-gate rate on joint-gap witnesses, while the joint certificate commits none; clean \(C\)-vs-\(R\) separation AUC is \(0.54\)--\(0.65\): the distinction lives in \(B_{d,\varepsilon}\), not at the point.}
\label{tab:comparators}
\resizebox{\textwidth}{!}{%
\begin{tabular}{lcc}
\toprule
\multicolumn{3}{l}{\emph{Does a joint-gap witness pass? (the witness is safe at the observed point)}} \\
\midrule
defense & evaluation locus & witness \\
\midrule
runtime enforcement (AgentSpec, VeriGuard) & observed point \((z,a)\)               & \textsc{allow}         \\
CaMeL (capability + IFC)                    & value provenance / trust               & \textsc{allow}\(^{*}\) \\
\method{} (ours)                            & neighborhood \(B_{d,\varepsilon}(z)\)  & \textsc{block}         \\
\midrule
\multicolumn{3}{l}{\emph{Orthogonality on the real finance schema}} \\
\midrule
cell / system & CaMeL & \method{} joint certificate \\
in-scope provenance, joint-gap witness & \textsc{admit} & \textsc{block} \\
injected (source-tainted) value, robust point & \textsc{block} & \textsc{allow} \\
\midrule
\multicolumn{3}{l}{\emph{A point-accurate classifier still commits}} \\
\midrule
 & transcript classifier (pt.\ acc.\ \(1.00\)) & \method{} joint certificate \\
\(\Pr[\text{unsafe exec}]\) on witnesses (fin/SRE) & \(1.00/1.00\) (= no gate) & \(0.000\) \\
\bottomrule
\end{tabular}}
\end{table*}

\medskip\noindent\textbf{Question.} Does non-composition realize as committed side effects end-to-end?
\textbf{Setting.} Full agent chain with committed effects (money approved / alert suppressed), \(300\) episodes per (domain, attack), finance and SRE.
\textbf{Result.} Under the joint-gap attack the composed marginal certificates commit unsafe effects at the full no-gate rate; the joint certificate commits none (Table~\ref{tab:e2e}).

\begin{table}[h]
    \centering
    \small
    \caption{End-to-end unsafe commits (Section~6.3): full chain with committed side effects (money approved / alert suppressed), \(300\) episodes per (domain, attack), finance/SRE. Under the joint-gap attack the analytic marginal certificate commits unsafe effects at the no-gate rate: composing two sound marginal certificates produces concrete unsafe commits, the non-composition thesis realized end-to-end. The joint certificate stays non-vacuous on clean traffic (privileged-execution rates \(0.23\) finance / \(0.117\) SRE, all safe).}
    \label{tab:e2e}
    \resizebox{\columnwidth}{!}{%
        \begin{tabular}{lcc}
            \toprule
            gate                          & \(\Pr[\text{unsafe exec}]\), joint-gap witness (fin/SRE) & mixed                      \\
            \midrule
            no gate                       & \(1.00\ /\ 0.40\)                                    & \(0.33\ /\ 0.24\)          \\
            point (learned)               & \(1.00\ /\ 0.40\)                                    & \(0.32\ /\ 0.14\)          \\
            marginal certificate          & \(1.00\ /\ 0.40\)                                    & \(0.32\ /\ 0.14\)          \\
            pre-execution                 & \(0.44\ /\ 0.23\)                                    & \(0.11\ /\ 0.13\)          \\
            joint certificate (\method{}) & \(\mathbf{0.00\ /\ 0.00}\)                           & \(\mathbf{0.00\ /\ 0.00}\) \\
            robust oracle                 & \(0.00\ /\ 0.00\)                                    & \(0.00\ /\ 0.00\)          \\
            \bottomrule
        \end{tabular}}
\end{table}

\medskip\noindent\textbf{Question.} Does the exploit and its defense hold on a live cluster?
\textbf{Setting.} A \texttt{kind} cluster (Kubernetes v1.31) with real Kyverno v1.13 admission; the tier is resolved from a registry ConfigMap (the TOCTOU surface) serving a stale \texttt{lax} tier while the true tier is \texttt{strict}.
\textbf{Result.} The over-provisioned workload is admissible under the served binding and blocked by the certified gate, with the matching safe control passing (Table~\ref{tab:realharness}).

\begin{table}[h]
\centering
\small
\caption{Real-harness integration (Experiments P3, P3-MCP, P3-mediation, P3-A8; Section~6.3). A live \texttt{kind} cluster (Kubernetes v1.31) with real Kyverno v1.13 admission resolves the tier from a separate registry ConfigMap (the TOCTOU surface); the registry serves a stale \texttt{lax} tier while the true tier is \texttt{strict}, so a six-replica workload is admissible under the served binding and unsafe under the true one. The certified gate (rung~1) enumerates the tier neighborhood, tests the \(\varepsilon\)-shifted cap on every branch, and blocks before the call; with a fresh registry Kyverno itself denies, matching the gate; a genuinely safe workload deploys (non-vacuous). Gate overhead \(\approx0.3\,\mu\)s per decision (\(\approx10\,\mu\)s on the MCP path, vs \(\approx2.5\)\,s LLM decode). Four real proposers (three Qwen2.5 variants and \texttt{qwen3.6}) all submit the unsafe apply through a real MCP write path. Integration illustrations only: the certified object remains \(\Allow(z,a)\).}
\label{tab:realharness}
\resizebox{\columnwidth}{!}{%
\begin{tabular}{llll}
\toprule
arm & ungated outcome & certified gate & ground truth / note \\
\midrule
P3 admission (stale registry) & unsafe Deployment created & blocked, no side effect & fresh-registry Kyverno denies \\
P3-MCP write (4 LLM proposers) & side effect \(4/4\) models & \(0/4\) & proposer stays uncertified \\
P3 mediation (create + scale paths) & scale path leaks (coverage \(0.5\)) & both blocked (coverage \(1.0\)) & deployed policy misses the subresource \\
P3-A8 filesystem MCP (quota) & \(1.5\)\,MB over-quota file written & blocked before the call & safe \(0.5\)\,MB write allowed \\
\bottomrule
\end{tabular}}
\end{table}

\medskip\noindent\textbf{Question.} Does a deployed-grade AML engine reproduce the taxonomy and the exploit?
\textbf{Setting.} The idiom \(\texttt{risk\_score}<\theta(\text{provenance})\) authored as a Marble~v1.4.0 scenario with the real IEEE-CIS constants; probes and end-to-end commits run through the engine's own decision API.
\textbf{Result.} The engine labels the four probes exactly as the taxonomy predicts, and every ungated approval is persisted by the engine while the certified run commits none (Tables~\ref{tab:marble}--\ref{tab:marblee2e}).

\begin{table}[h]
\centering
\small
\caption{Purpose-built AML-engine verification of the joint-gap taxonomy (B2-Marble; Section~6.2). The idiom \(\texttt{risk\_score}<\theta(\text{provenance})\) is authored as a Marble~v1.4.0 scenario with the real IEEE-CIS constants and published; real transactions are run through the engine's own decision API, which labels the four probes (clean, \(d{=}1\) provenance swap, \(+\varepsilon\), joint swap\(+\varepsilon\)) per record. Claim discipline: a real deployed-grade AML engine executing an \emph{authored} rule on real data: engine-verified existence, one level above the general-purpose engines (OPA, GoRules ZEN), not a claim the rule was mined from a bank's deployed policies.}
\label{tab:marble}
\resizebox{\columnwidth}{!}{%
\begin{tabular}{lc}
\toprule
real IEEE-CIS transactions (decisions) & \(800\) (\(3200\)) \\
engine-labelled joint-gap witnesses & \(160\) (\(20\%\)) \\
engine\(\leftrightarrow\)analytic agreement & \(1.0\) \\
\(C\)-set Jaccard & \(1.0\) \\
per-witness engine pattern & approve/approve/approve/\textsc{review} \\
\bottomrule
\end{tabular}}
\end{table}

\begin{table}[h]
\centering
\small
\caption{Committed side effects on a deployed-grade AML engine (Experiment B2-Marble e2e; Section~6.3). The pipeline serves the stale loose provenance; the gate certifies \(\Ball(z)\) around the served return; on allow the adversary realizes the in-budget worst case. Every ungated approval is committed and persisted by Marble's own decision API (verified via \texttt{GET /v1/decisions}); the certified gate submits none of them and stays non-vacuous; Marble itself reviews every realized worst case, agreeing with the gate. \(100\) joint-gap witnesses + \(100\) robust-safe controls, real IEEE-CIS transactions, \(\varepsilon=0.10\).}
\label{tab:marblee2e}
\resizebox{0.95\columnwidth}{!}{%
\begin{tabular}{lcc}
\toprule
arm & unsafe approvals committed & robust-safe allowed \\
\midrule
engine only (served provenance) & \(100/100\) (persisted) & --- \\
certified gate & \(\mathbf{0/100}\) & \(100/100\) \\
oracle (worst case submitted) & reviewed \(40/40\) by Marble & --- \\
\bottomrule
\end{tabular}}
\end{table}

\medskip\noindent\textbf{Question.} Does a second, independent real adapter reproduce the taxonomy?
\textbf{Setting.} Kubernetes Deployment manifests mapped to \(z=(\texttt{tier},\texttt{cost})\), with real Kyverno admission enforcing \(\texttt{cost}\le\theta(\texttt{tier})\) from a ConfigMap.
\textbf{Result.} A different domain, return format, and threshold family reproduce the same A/B/C/R/U structure (Table~\ref{tab:a7}).

\begin{table}[h]
\centering
\small
\caption{Second independent real adapter (Experiment A7; Section~6.2). A Kubernetes Deployment manifest maps to \(z=(\texttt{tier},\texttt{cost})\) and real Kyverno admission enforces \(\texttt{cost}\le\theta(\texttt{tier})\) with the threshold resolved from a ConfigMap (\(\theta_{\mathrm{strict}}=0.50\), \(\theta_{\mathrm{loose}}=0.80\), \(\varepsilon=0.10\)): a different domain, return format, and engine from the IEEE-CIS adapters. The engine admits every single-channel move and denies the joint one, so naive marginal composition false-allows on the engine itself.}
\label{tab:a7}
\resizebox{0.95\columnwidth}{!}{%
\begin{tabular}{ll}
\toprule
quantity & value \\
\midrule
taxonomy over the cost channel & U \(39\) / A \(61\) / C \(20\) / R \(80\) (C band \([0.40,0.50)\)) \\
engine-vs-analytic agreement & \(1.0\) (\(18\) real admission checks) \\
real Kyverno joint-gap witnesses & \(6\): admit clean, swap-only, \(+\varepsilon\); deny joint \\
\bottomrule
\end{tabular}}
\end{table}

\FloatBarrier
\subsection{External-validity evidence}
The frozen-scan evidence hierarchy and the third-party deployed-rule anchor.

\medskip\noindent\textbf{Question.} What exactly is claimed at each level of external validity?
\textbf{Setting.} An evidence hierarchy kept separate by provenance: third-party existence, regulatory-grounded continuous mechanism, and engine-validated defense.
\textbf{Result.} Existence of the idiom is anchored in third-party code and quantized cloud policy; the continuous mechanism uses source-locked regulatory thresholds; deployment prevalence is explicitly not claimed (Table~\ref{tab:claimladder}).

\begin{table}[h]
    \centering
    \scriptsize
    \setlength{\tabcolsep}{4pt}
    \caption{Evidence hierarchy for external validity (Section~6.4), kept separate by provenance. Existence of the category-conditioned-threshold idiom is anchored in third-party code; the continuous mechanism and the certified defense use source-locked regulatory thresholds and an executable policy-as-code engine, respectively. Per-corpus counts behind Section~6.4: \(\approx 1424\) k8s/cloud policies (39 numeric thresholds, 0 provenance-keyed); 43 MCP reference tools (29 untyped, substrate 0); 4138 OpenAPI specs / 701 providers (structural co-occurrence \(9.8\%\), finance habitat \(22.4\%\), confirmed pipeline-set 0); MCP registries: 1085 servers / 5002 tools, substrate \(5.0\%\) of typed returns (Wilson 95\% \([3.5,7.0]\%\), threat-intel/content/bibliographic); DMN-TCK \(4.9\%\) / Kogito \(11.5\%\); OpenFisca \(6.8\%\) (subject-keyed).}
    \label{tab:claimladder}
    \resizebox{\columnwidth}{!}{%
        \begin{tabular}{>{\raggedright\arraybackslash}p{0.24\linewidth}>{\raggedright\arraybackslash}p{0.30\linewidth}>{\raggedright\arraybackslash}p{0.24\linewidth}>{\raggedright\arraybackslash}p{0.24\linewidth}}
            \toprule
            claim                              & evidence                                                                                                                      & provenance                                         & caveat                                                                        \\
            \midrule
            third-party existence (quantized)  & Azure Key Vault \texttt{keyType} \(\to\) \texttt{keySize} (7 instances, idiom rate \(0.667\))                                 & third-party logic re-implemented                   & \texttt{keySize} quantized: existence anchor rather than continuous substrate \\
            third-party existence (continuous) & OpenFisca \(6.8\%\); DMN tables \(4.9\)--\(11.5\%\) incl.\ OMG spec example; first-class in ZEN/Tazama/Marble; k8s/cloud null & third-party executable, frozen-protocol scan & subject-keyed rather than pipeline-set: security-relevant rate \(0\)          \\
            continuous mechanism               & PSD2 / AML source-locked thresholds; natural \(C\) \(6.5\)--\(9.8\%\) at \(\varepsilon=0.10\)                                 & regulatory-grounded authored policy                & executable policy authored from cited thresholds                              \\
            engine-verified \(C\) on real data & OPA \& ZEN over real IEEE-CIS: \(800\) witnesses, engine-vs-analytic agreement \(1.000\)                                      & authored rule, two production engines              & general-purpose engines rather than a deployed AML substrate                  \\
            certified defense                  & OPA/Rego (Track C): \(C\approx10\)--\(12\%\), \(\Callow=\Uallow=0\), \(\texttt{cert\_false\_allow}=0\)                        & authored provenance-conditioned Rego               & node-level, learned-gate certificate                                          \\
            \bottomrule
        \end{tabular}}
\end{table}

\medskip\noindent\textbf{Question.} Does the candidate substrate survive conservative adjudication?
\textbf{Setting.} Two-pass Stage-2 adjudication (lexical + frozen semantic table, disagreement \(\Rightarrow\) OUT) of the \(31\) Stage-1 candidates across \(8\) servers; zero execution.
\textbf{Result.} The \(5.0\%\) candidate habitat collapses to a fourth informative null; the structural residual is named rather than hidden (Table~\ref{tab:adjudication}).

\begin{table}[h]
\centering
\small
\caption{Registry-scale substrate adjudication (EXP-A2; Section~6.4). Conservative two-pass Stage-2 adjudication of the 31 Stage-1 candidates (8 servers, frozen detector): the $5.0\%$ candidate habitat collapses to a fourth informative null; the structural residual is named, not hidden. Zero-execution (published schemas).}
\label{tab:adjudication}
\resizebox{\columnwidth}{!}{%
\begin{tabular}{lll}
\toprule
stage & count & rate of typed returns (Wilson 95\%) \\
\midrule
Stage-1 candidates & $31/622$ tools, $8$ servers & $5.0\%$ $[3.5,7.0]\%$ \\
two-pass structural survivors & $1/31$ (\texttt{cache\_respected}) & $0.16\%$ $[0.03,0.9]\%$ \\
documented $\theta(s)$ & $0/31$ & $0.0\%$ $[0,0.6]\%$ \\
\bottomrule
\end{tabular}}
\end{table}

\medskip\noindent\textbf{Question.} How prevalent is the structural idiom in committed third-party rule code?
\textbf{Setting.} The frozen Phase-1 detector applied through additional format parsers to legislation-as-code, decision-table, and rule-engine corpora.
\textbf{Result.} Structural presence reaches \(4.9\)--\(11.5\%\), but in every hit the selector is a subject/status attribute: \(\texttt{provenance\_upstream\_rate}=0\) (Table~\ref{tab:idiomaticity}).

\begin{table}[h]
\centering
\small
\caption{Structural prevalence and per-engine idiomaticity of \(\mathrm{op}(f_{\mathrm{num}},\theta(s))\) in third-party committed rule code (Experiments P1-B, A-DMN, B1; frozen Phase-1 detector, detector integrity hashes recorded). \emph{Top:} structural rates in committed corpora; in every hit the selector \(s\) is a subject/status attribute, so \(\texttt{provenance\_upstream\_rate}=0\) (structural presence, distinct from security-relevant prevalence). \emph{Bottom:} per-engine expressibility inventory (no rate emitted): the idiom is a first-class construct in decision-table engines and purpose-built AML engines, expressible-but-generic elsewhere.}
\label{tab:idiomaticity}
\resizebox{\columnwidth}{!}{%
\begin{tabular}{lll}
\toprule
committed corpus & structural idiom rate & selector semantics \\
\midrule
OpenFisca-France models & \(6.8\%\) (\(9/132\)) & subject (zone, statut) \\
DMN-TCK decision tables & \(4.9\%\) (\(8/162\); incl.\ the OMG lending example) & subject \\
Kogito examples & \(11.5\%\) (\(6/52\)) & subject \\
k8s/cloud admission (control) & \(0\) (39 constant thresholds) & --- \\
\midrule
engine & idiom first-class? & evidence \\
\midrule
DMN/FEEL; GoRules ZEN (JDM) & yes & decision-table grammar \\
Tazama & yes & \texttt{RuleConfig} bands, on-disk-verified \\
Marble & yes & score-banding, public docs \\
OPA/Rego; Drools; Jube & partial & expressible, no dedicated construct \\
\bottomrule
\end{tabular}}
\end{table}

\medskip\noindent\textbf{Question.} Does a verbatim third-party policy instantiate the mechanism?
\textbf{Setting.} OpenFisca-France Bail R\'eel Solidaire income ceilings \(\theta(\mathrm{zone},\mathrm{household})\) at a frozen commit, used unmodified; \(\Ndisc\) = the policy's own five zones, \(n=20\)k.
\textbf{Result.} The point gate admits all \(1{,}773\) joint-gap witnesses and the exact certificate admits none; the zone key is subject-set, so the claim is scoped accordingly (Table~\ref{tab:cx5openfisca}).

\begin{table}[h]
\centering
\small
\caption{Independent third-party policy case study (EXP-CX5; Section~6.4). OpenFisca-France, Bail R\'eel Solidaire income ceilings \(\theta(\mathrm{zone},\mathrm{household})\) (frozen commit \texttt{a9d8dcbe}, transcribing Arr\^et\'e du 11/12/2023 / Art.~R255-1 CCH); the rule is used verbatim, \(\Ndisc\) = its own five zones, \(n=20\)k. Scope: zone is a subject/region attribute, not a pipeline-provenance key: this is a deployed-threshold \emph{existence} anchor, complementary to the pipeline-provenance nulls (S24, S25).}
\label{tab:cx5openfisca}
\resizebox{\columnwidth}{!}{%
\begin{tabular}{lcc}
\toprule
 & \(\varepsilon=0.03\) & \(\varepsilon=0.10\) \\
\midrule
natural \(\Pr[C]\) (joint-gap witnesses) & \(0.035\) & \(0.089\) \\
point gate grants on the witness set & \multicolumn{2}{c}{\(1773/1773\) (\(\varepsilon=0.10\))} \\
exact certificate grants on the witness set & \multicolumn{2}{c}{\(0\)} \\
\(\Rallow\) (non-vacuous) & \(0.34\) & \(0.22\) \\
real zone gaps (1--6 person household) & \multicolumn{2}{c}{EUR~5{,}029 -- EUR~34{,}543 (normalized \(\delta\) \(0.042\)--\(0.286\))} \\
\bottomrule
\end{tabular}}
\end{table}

\FloatBarrier
\subsection{Threat-model calibration and residual risk}
Budget calibration, discrete-channel and freshness escape, and negative controls.

\medskip\noindent\textbf{Question.} Do correlated faults exceed the declared discrete budget?
\textbf{Setting.} Compound fault injection (pairs and triples) with realized flip mass measured per clean category (EXP-A1).
\textbf{Result.} All fault pairs land at \(d\le2\) (\(\mathrm{frac}_{d\le2}=1.000\)) and the \(d=2\) Lipschitz gate stays sound with \(\Rallow=0.40\)--\(0.75\) (Table~\ref{tab:compound}).

\begin{table}[h]
\centering
\small
\caption{Compound / correlated fault injection (EXP-A1; Table~\ref{tab:faults} and Section~6.5). All fault \emph{pairs} land at $d\le2$ ($\mathrm{frac}_{d\le2}=1.000$); the $d=2$ Lipschitz gate stays sound ($\texttt{cert\_false\_allow}=0$, $f_{\mathrm{scale}}=3$, $\Rallow=0.40$--$0.75$). Realized C$\to$U flip mass is measured per clean category (e.g.\ finance \texttt{stale\_cache}+\texttt{wrong\_provenance}: $13.6\%$ of clean-$C$ records, adversarial regime). $4000$ samples/combo $\times$ 3 seeds, real IEEE-CIS + finance/SRE.}
\label{tab:compound}
\resizebox{\columnwidth}{!}{%
\begin{tabular}{llll}
\toprule
combo class & $d$ reached & $\varepsilon$ behavior & $\Pr$ (adversarial / independent) \\
\midrule
two distinct discrete (prov+policy, prov+env) & $d=2$ & unchanged & $1.0$ / $\approx0.05$ \\
discrete + continuous & $d\le1$ & larger, mostly $\le\varepsilon$ & --- \\
two continuous (stale+collision; skew+stale) & $d=0$ & $\varepsilon_{p95}\approx1.1$; $\approx0.8$ (cliff) & --- \\
three distinct discrete (prov+policy+env) & $d=3$ & --- & $0.80$/$0.50$ IEEE/fin.\ ; $0.001$--$0.004$ \\
\bottomrule
\end{tabular}}
\end{table}

\medskip\noindent\textbf{Question.} What does raising the discrete budget cost?
\textbf{Setting.} \(d\in\{1,2,3\}\) with exact \(\Ndisc\) enumeration, deterministic Lipschitz backend, both tracks, \(3\) seeds.
\textbf{Result.} Sound at every \(d\) on both tracks with graceful utility decay; the numeric scaling is held-out validated (Table~\ref{tab:dsweep}).

\begin{table}[h]
    \centering
    \small
    \caption{Discrete-budget sweep \(d\in\{1,2,3\}\) (Section~6.5), exact \(\Ndisc\) enumeration, deterministic Lipschitz backend, both tracks numeric-scaled (OPA \(\texttt{fscale}=3\), synthetic \(4\): the largest values keeping \(\texttt{cert\_false\_allow}=0\) at every \(d\); held-out selection hygiene in Table~\ref{tab:fscaleheldout}); \(3\) seeds. Sound at every \(d\) on both tracks with graceful utility decay; the RS ablation instead collapses through its per-branch correction \(\alpha/|\Ndisc|\) (OPA \(0.05\to0\to0\); synthetic \(0.39\to0.09\)). \(|\Ndisc|\) grows \(8\to24\to36\) (OPA) and \(18\to109\to256\) (synthetic \(|X_1|{=}8\)), cost ratio \(1/4.7/9.9\); the operational enumeration cliff is \(d\approx 2\), where \(|\Ndisc|\) reaches \(24\)--\(109\) across tracks. OPA \(\texttt{fscale}\ge 4\) lifts \(\Rallow\) but breaks soundness at \(d\ge 2\).}
    \label{tab:dsweep}
    \resizebox{0.95\columnwidth}{!}{%
        \begin{tabular}{lcccc}
            \toprule
            track                   & \(d{=}1\) & \(d{=}2\) & \(d{=}3\) & \(\texttt{cert\_false\_allow}\) \\
            \midrule
            OPA finance             & \(0.587\) & \(0.507\) & \(0.453\) & \(0\)                           \\
            synthetic \(|X_1|{=}4\) & \(0.88\)  & \(0.79\)  & \(0.79\)  & \(0\)                           \\
            synthetic \(|X_1|{=}8\) & \(0.85\)  & \(0.80\)  & \(0.80\)  & \(0\)                           \\
            \bottomrule
        \end{tabular}}
\end{table}

\medskip\noindent\textbf{Question.} Which fault mechanisms escape the discrete neighborhood?
\textbf{Setting.} Leave-one-fault-out: rebuild \(\Ndisc\) without each mechanism and measure the escape probability on real substrates (\(3\) seeds, \(\varepsilon=0.10\)).
\textbf{Result.} Provenance, policy-pack, and TOCTOU faults never escape; the two escaping faults have no discrete footprint and are exactly what the validation stack targets (Table~\ref{tab:escape}).

\begin{table}[h]
    \centering
    \small
    \caption{Discrete-channel escape, leave-one-fault-out (Section~6.5): rebuild \(\Ndisc\) without each mechanism and measure \(\Pr[z'\notin\Ndisc^{-M}\times B_\varepsilon]\) (real IEEE-CIS / finance / SRE substrates, \(3\) seeds, \(\varepsilon=0.10\)). Provenance, policy-pack, and TOCTOU faults never escape; the two escaping faults have no discrete footprint (they drift \(x\) past \(\varepsilon\), the out-of-budget tail scoped to validation). Over-declaring \(K\) inert branches costs utility: \(\Rallow=0.172/0.109/0.051/0.012/0.002\) at \(K=0/1/2/4/8\) (\(|\Ndisc|\ 3\to 11\)).}
    \label{tab:escape}
    \resizebox{\columnwidth}{!}{%
        \begin{tabular}{lccc}
            \toprule
            held-out mechanism       & IEEE-CIS  & finance   & SRE       \\
            \midrule
            wrong provenance binding & \(0.000\) & \(0.000\) & \(0.000\) \\
            wrong policy pack        & \(0.000\) & \(0.000\) & \(0.000\) \\
            TOCTOU environment label & \(0.000\) & \(0.000\) & \(0.000\) \\
            schema skew              & \(0.811\) & \(0.861\) & \(0.865\) \\
            cache key collision      & \(0.982\) & \(1.000\) & \(0.999\) \\
            \bottomrule
        \end{tabular}}
\end{table}

\medskip\noindent\textbf{Question.} What continuous residual does each covered fault carry?
\textbf{Setting.} \(4000\) samples per fault, measured after the integrity+freshness validation stack (Experiment FAULT).
\textbf{Result.} Discrete binding faults carry no continuous drift; the pooled \(\varepsilon_{\mathrm{emp}}@\mathrm{p95}\approx0.079\)--\(0.095\) anchors Table~\ref{tab:faults} (Table~\ref{tab:perfault}).

\begin{table}[h]
\centering
\small
\caption{Per-fault continuous residuals (Experiment FAULT, \(4000\) samples/fault; anchors Table~\ref{tab:faults}'s pooled \(\varepsilon_{\mathrm{emp}}@\mathrm{p95}\approx0.079\)--\(0.095\), measured after the integrity+freshness validation stack). Discrete binding faults carry no continuous drift; the two out-of-budget mechanisms are exactly those a schema/identity validation layer targets.}
\label{tab:perfault}
\resizebox{\columnwidth}{!}{%
\begin{tabular}{llll}
\toprule
fault & \(d\) & continuous drift (pre-validation) & in \(B_{1,0.10}\) \\
\midrule
provenance / policy-pack / TOCTOU & \(1\) (\(\Pr=1.000\)) & \(\varepsilon=0\) & yes \\
numeric jitter                    & \(0\) & \(\texttt{frac\_in\_B}=0.998\) & yes \\
normalization skew                & \(0\) & \(\texttt{frac\_in\_B}=0.974\) & yes \\
stale cache (same surface)        & \(0\) & \(\varepsilon_{p90}\approx0.25\), \(\varepsilon=0.10\) covers \(\approx64\%\) & partial \\
schema transposition              & \(0\) & \(\texttt{frac\_in\_B}\approx0.13\)--\(0.18\) & no (validation failure) \\
cache key collision               & \(0\) & \(\texttt{frac\_in\_B}\approx0.00\)--\(0.01\) & no (validation failure) \\
\bottomrule
\end{tabular}}
\end{table}

\medskip\noindent\textbf{Question.} What does the certificate explicitly not cover?
\textbf{Setting.} Fabricated returns with gap \(\le\varepsilon\) versus gap \(\gg\varepsilon\) (negative control).
\textbf{Result.} The in-budget lie is refused; far outside the budget the certificate makes no claim and a sound gate allows the fabricated return --- the boundary of the guarantee, stated (Table~\ref{tab:negcontrol}).

\begin{table}[h]
    \centering
    \scriptsize
    \setlength{\tabcolsep}{4pt}
    \caption{Negative control (Appendix~E.4). For \(\texttt{gap}\le\varepsilon\) the lie is inside the budget and the gate refuses; for \(\texttt{gap}\gg\varepsilon\) the certificate makes no claim and a sound gate allows the fabricated return.}
    \label{tab:negcontrol}
    \resizebox{\columnwidth}{!}{%
        \begin{tabular}{lccccc}
            \toprule
            \(\texttt{fabrication\_gap}\) & dist.\ to true & inside budget & gate allows fake & claim applicable & unsafe if endpoint lies \\
            \midrule
            0.05                          & 0.05           & yes           & 0.00             & 1.00             & 0.00                    \\
            0.10 (\(=\varepsilon\))       & 0.10           & yes           & 0.00             & 1.00             & 0.00                    \\
            0.20                          & 0.20           & no            & 0.00             & 0.00             & 0.00                    \\
            0.40                          & 0.40           & no            & 0.283            & 0.00             & 0.283                   \\
            0.80                          & 0.678          & no            & 1.00             & 0.00             & 1.00                    \\
            \bottomrule
        \end{tabular}}
\end{table}

\medskip\noindent\textbf{Question.} Where does the declared budget actually break?
\textbf{Setting.} The model-free robust-oracle certificate for \(B_{1,0.10}\) held fixed while the adversary moves strictly outside it (Experiment P2).
\textbf{Result.} False allow over certified allows is \(0\) at every in-budget point and first leaves zero at \(\varepsilon^{*}=0.11\), \(d^{*}=2\) (Table~\ref{tab:breaking}).

\begin{table}[h]
\centering
\small
\caption{Breaking radius of the declared budget (Experiment P2; Section~6.5). The model-free robust-oracle certificate for \(B_{1,0.10}\) is held fixed while the adversary moves strictly outside it; the false-allow rate over certified allows is \(0\) at every in-budget point and first leaves zero at \(\varepsilon^{*}=0.11\) and \(d^{*}=2\) (\(d=3\) adds nothing: the schema exhausts its discrete atoms at \(d=2\)): degradation is graceful and visible, never a silent in-budget false allow. Mechanism placement matches Table~\ref{tab:faults}: the in-budget provenance fault leaks \(0.000\); stale cache \(0.014/0.018\); schema skew \(0.089/0.028\); cache key collision (\(\varepsilon_{p95}\approx1.2\), \(12\times\) the budget) \(0.325/0.244\). Scope probe (Experiment SEL): a metadata-poisoning mis-selection within the gated action group is a \(d=1\) provenance swap the certificate already covers (unsafe execution \(1.00\to0.00\)); a cross-action mis-selection lies outside the per-action ball and retains a \(0.31\) residual: tool selection stays an upstream defense surface.}
\label{tab:breaking}
\resizebox{0.9\columnwidth}{!}{%
\begin{tabular}{lcc}
\toprule
probe & finance & SRE \\
\midrule
in-budget (\(d{=}1\), \(\varepsilon\le0.10\)) & \(0.000\) & \(0.000\) \\
\(\varepsilon^{*}=0.11\) (first leak) & \(0.046\) & \(0.083\) \\
\(\varepsilon=0.50\) & \(0.98\) & \(0.998\) \\
\(d=2\) at \(\varepsilon=\varepsilon_{\mathrm{cert}}\) & \(0.126\) & \(0.297\) \\
\bottomrule
\end{tabular}}
\end{table}

\medskip\noindent\textbf{Question.} What happens when the trusted constructor itself is corrupted?
\textbf{Setting.} A field flip with probability \(p\) at the \(z\)-constructor, below the typed interface (\(5\) seeds).
\textbf{Result.} False allow grows monotonically with \(p\), delimiting where the guarantee stops; this anchors the \(1.7\%\) trust-boundary number (Table~\ref{tab:constructor}).

\begin{table}[h]
\centering
\small
\caption{Constructor-corruption sweep (EXP2, part B; anchors the \(1.7\%\) trust-boundary number). A field flip with probability \(p\) at the \(z\)-constructor, below the typed interface; the verified-point gate trusts the delivered provenance. False allow grows monotonically with \(p\), delimiting where the guarantee stops; \(5\) seeds.}
\label{tab:constructor}
\resizebox{\columnwidth}{!}{%
\begin{tabular}{lccccc}
\toprule
flip probability \(p\) & \(0\) & \(0.05\) & \(0.10\) & \(0.20\) & \(0.40\) \\
\midrule
verified-point false allow & \(0.000\) & \(0.0025\) & \(0.0047\) & \(0.0077\) & \(0.0167\) \\
\bottomrule
\end{tabular}}
\end{table}

\medskip\noindent\textbf{Question.} What freshness SLA keeps the declared budget valid?
\textbf{Setting.} Same-card wall-clock staleness on real IEEE-CIS: a fine sub-minute grid \(\Delta t\in\{1,\dots,120\}\)\,s (\(5\) seeds, coverage-aware) and the coarse grid plotted in Figure~4.
\textbf{Result.} Certified false allow is \(0\) at every \(\Delta t\), but no well-covered sub-minute SLA meets a \(<1\%\) system escape: scores must be recomputed in-loop (Tables~\ref{tab:slasweep}--\ref{tab:slacoarse}).

\begin{table}[h]
\centering
\small
\caption{Sub-minute freshness-SLA sweep (EXP-A3; Section~6.5). Fine grid $\Delta t\in\{1,\dots,120\}$\,s on real IEEE-CIS same-card wall-clock staleness, 5 seeds, coverage-aware. $\texttt{cert\_false\_allow}=0$ at every $\Delta t$; no well-covered sub-minute SLA meets a $<1\%$ system escape, so scores must be recomputed in-loop.}
\label{tab:slasweep}
\resizebox{\columnwidth}{!}{%
\begin{tabular}{llll}
\toprule
$\Delta t$ band & coverage & $\varepsilon_{\mathrm{emp}}$@p95 & system false-allow \\
\midrule
$\le1$\,s & $\approx0.5$ re-reads (artifact) & below $\varepsilon$ & $<0.01$ (zero-coverage) \\
$\approx10$\,s & low & crosses $0.10$ & --- \\
$\ge10$\,s & growing & $0.13$--$0.17$ & --- \\
$30$--$90$\,s & $\ge30$ serves (well-covered) & $0.13$--$0.17$ & $0.03$--$0.05$ \\
\bottomrule
\end{tabular}}
\end{table}

\begin{table}[h]
\centering
\small
\caption{Coarse freshness-SLA grid (Experiment EXP2-A; the points Figure~4 plots). Same-card wall-clock staleness on real IEEE-CIS; \(\texttt{cert\_false\_allow}=0\) at every \(\Delta t\); the fine sub-minute grid is Table~S35.}
\label{tab:slacoarse}
\resizebox{\columnwidth}{!}{%
\begin{tabular}{lcccccc}
\toprule
\(\Delta t\) & 60\,s & 30\,min & 1\,h & 6\,h & 1\,d & 120\,d \\
\midrule
\(\varepsilon_{\mathrm{emp}}\)@p95 & \(0.174\) & \(0.239\) & \(0.267\) & \(0.292\) & \(0.298\) & \(0.306\) \\
system false-allow                 & \(0.018\) & \(0.034\) & \(0.040\) & \(0.058\) & \(0.062\) & \(0.071\) \\
\bottomrule
\end{tabular}}
\end{table}

\medskip\noindent\textbf{Question.} What does the normalized \(\varepsilon=0.10\) mean in raw units?
\textbf{Setting.} The raw-unit move corresponding to a normalized \(\varepsilon\) step, per field and anchor.
\textbf{Result.} Per-field raw-unit equivalents (marginal view; the joint \(\ell_2\) ball couples fields, so single-field values are upper bounds) (Table~\ref{tab:rawunit}).

\begin{table}[h]
\centering
\small
\caption{Raw-unit $\varepsilon$ audit (EXP-B2; Section~6.5 and Appendix~D). The raw-unit move corresponding to a normalized $\varepsilon=0.10$ step, per field and anchor (marginal per-field view; the joint $\ell_2$ ball couples fields, so single-field values are upper bounds).}
\label{tab:rawunit}
\begin{tabular}{llll}
\toprule
field (scaling) & p50 & p95 & p99 \\
\midrule
TransactionAmt (log) & $\approx\$71$ & $\approx\$453$ & $\approx\$557$ \\
NAB CPU (linear) & \multicolumn{3}{l}{$10$ CPU points (constant)} \\
IEEE C-aggregate (clip) & \multicolumn{3}{l}{$\pm10.7$ (constant)} \\
\texttt{risk\_score} & \multicolumn{3}{l}{$0.10$ probability points} \\
\bottomrule
\end{tabular}
\end{table}

\medskip\noindent\textbf{Question.} Do calibrated per-field budgets beat one isotropic \(\varepsilon\)?
\textbf{Setting.} Multivariate-affine policy (\(k=6\)), exact certificates with each geometry's dual norm, every geometry calibrated on a frozen split to the same held-out fault coverage (\(q=0.95\)), \(5\) seeds.
\textbf{Result.} The comparison is coverage-matched by construction; the residual \(\approx0.038\) policy false allow is the calibration escape --- the \(5\%\) fault tail beyond the frozen quantile (Table~\ref{tab:cx4perfield}).

\begin{table}[h]
\centering
\small
\caption{Calibrated per-field budget (EXP-CX4; Section~6.5). Multivariate-affine policy (\(k=6\)), exact certificate with the dual norm of each geometry; every geometry calibrated on a frozen split to the same held-out joint fault coverage (\(q=0.95\)), evaluated disjointly, 5 seeds. The residual \(\approx0.038\) policy false-allow is the calibration escape (the \(5\%\) fault tail beyond the frozen budget), identical across geometries: not certificate unsoundness. Real IEEE-CIS anchor: per-field p95 residual heterogeneity \(2.9\times\) (\(0.055\) vs \(0.019\)).}
\label{tab:cx4perfield}
\resizebox{\columnwidth}{!}{%
\begin{tabular}{lcc}
\toprule
budget geometry & held-out fault coverage & certified autonomy \(\Rallow\) \\
\midrule
global \(\ell_2\) ball        & \(\approx0.95\) & \(0.45\) \\
weighted-\(\ell_\infty\) box  & \(\approx0.95\) & \(0.48\) \\
diagonal ellipsoid            & \(\approx0.95\) & \(\mathbf{0.62}\) \\
\bottomrule
\end{tabular}}
\end{table}

\medskip\noindent\textbf{Question.} Does budget calibration survive a real adapter into a real engine?
\textbf{Setting.} The Table~\ref{tab:faults} fault mechanisms run through a real adapter into Marble's decision API; \(\varepsilon_{\mathrm{cal}}=0.052\) is the p95 residual on a calibration half (Experiment CX6).
\textbf{Result.} The calibrated budget holds held-out, leaking only outside the ball (Table~\ref{tab:cx6}).

\begin{table}[h]
\centering
\small
\caption{Real-adapter budget calibration through the Marble engine (Experiment CX6; Section~6.5). The Table~\ref{tab:faults} fault mechanisms run through a real adapter into Marble's decision API; \(\varepsilon_{\mathrm{cal}}=0.052\) is the p95 residual under integrity+freshness faults on a calibration half, and escape (a fault that flips the engine's decision yet falls outside \(B_{1,\varepsilon_{\mathrm{cal}}}\)) is measured on a disjoint holdout (\(n=800\) real IEEE-CIS transactions, \(1658\) unique engine decisions). The calibrated budget generalizes; the schema/identity tail is the measured out-of-budget escape, scoped to the validation layer.}
\label{tab:cx6}
\resizebox{0.95\columnwidth}{!}{%
\begin{tabular}{lcc}
\toprule
mechanism (holdout) & \(\varepsilon_{p95}\) & budget escape \\
\midrule
wrong provenance binding (\(d{=}1\)) & \(0\) & \(0.000\) \\
numeric jitter / normalization skew & within budget & \(\approx0\) \\
stale cache (same surface) & \(0.105\) & \(0.020\) \\
schema skew & \(0.53\) & \(0.103\) \\
cache key collision & \(0.42\) & \(0.31\) \\
\midrule
pooled: integrity+freshness vs out-of-budget tail & --- & \(0.006\) vs \(0.207\) \\
\bottomrule
\end{tabular}}
\end{table}

\medskip\noindent\textbf{Question.} Does typed-state projection restore fidelity at high ambient dimension?
\textbf{Setting.} A scalar policy on \(k_{\mathrm{active}}=5\) fields padded with nuisance dimensions to \(k_{\mathrm{raw}}\in\{20,50,100\}\); four gates differing only in the fields they may use, same smoothed certificate, scored against the true oracle.
\textbf{Result.} The dense gate reproduces the \(k=100\) fidelity edge; the projected gate removes it (Table~\ref{tab:projection}).

\begin{table}[h]
\centering
\small
\caption{Policy-state projection (Experiment D.4; Section~6.5). A scalar policy depends on \(k_{\mathrm{active}}=5\) fields, padded with nuisance dimensions to \(k_{\mathrm{raw}}\in\{20,50,100\}\); four gates differ only in the fields they may use, with the same smoothed certificate, scored against the true oracle. The dense gate reproduces the \(k=100\) fidelity edge of Table~\ref{tab:dimsweep}; projecting to the low-dimensional policy state removes it. The certifiable interface is a typed low-dimensional policy state rather than the raw record: the recommendation is projection, not smoothing the raw space.}
\label{tab:projection}
\resizebox{\columnwidth}{!}{%
\begin{tabular}{lcc}
\toprule
gate (fields used) & fidelity (\(k_{\mathrm{raw}}{=}20/50/100\)) & \(\texttt{cert\_false\_allow}\) (\(k{=}100\)) \\
\midrule
dense (all raw fields) & \(0.933/0.850/0.742\) & \(0.017\) \\
oracle projection (\(5\) active fields) & \(0.975/0.992/0.983\) & \(0\) \\
L1-bottleneck (\(\approx8\)--\(25\) fields) & \(\ge0.917\) & \(0\) \\
noise-trained dense & \(0.867\) at \(k{=}100\) & \(0\) \\
\bottomrule
\end{tabular}}
\end{table}

\medskip\noindent\textbf{Question.} Can abstention itself be weaponized (availability, not safety)?
\textbf{Setting.} Boundary-seeking input selection against the exact analytic certificate on real IEEE-CIS (\(3\) seeds).
\textbf{Result.} The attack inflates abstention but never safety: certified false allow stays \(0\) under benign and every adversarial strength; rate-limit mitigations are noted (Table~\ref{tab:dos}).

\begin{table}[h]
    \centering
    \small
    \caption{Abstention-DoS (Section~7): boundary-seeking input selection against the exact analytic certificate on real IEEE-CIS (\(3\) seeds; \(\theta_{\mathrm{base}}=0.4888\), \(\delta=0.08\), \(\varepsilon=0.10\)). The attack inflates abstention (availability) but never safety: \(\texttt{cert\_false\_allow}=0\) under benign and every adversarial strength. Mitigations: a per-source rate limit caps abstention at \(0.37\) (cost: \(0.84\) of adversarial volume dropped); adaptive-\(\varepsilon\) reclaims availability (\(0.99\to0.75\) abstention at \(\varepsilon=0.02\)) at residual exposure \(0.24\); \(\texttt{cert\_false\_allow}=0\) under both. The boundary-balanced pool makes the benign floor conservative.}
    \label{tab:dos}
    \resizebox{\columnwidth}{!}{%
        \begin{tabular}{lcc}
            \toprule
            selection                               & abstention                  & inflation                \\
            \midrule
            benign (deep interior)                  & \(0.406\pm0.003\)           & \(1.00\times\)           \\
            boundary-seeking (steer \(0.25\to1.0\)) & \(0.552/0.698/0.843/0.990\) & \(1.36\)--\(2.44\times\) \\
            \bottomrule
        \end{tabular}}
\end{table}

\FloatBarrier
\subsection{Reproducibility and statistical audits}
Cost, monitoring, selection hygiene, horizon corrections, and the zero-cell audit.

\medskip\noindent\textbf{Question.} What does each backend cost per decision, and what does the cost buy?
\textbf{Setting.} Per-decision latency on the reference machine of Appendix~D, pointwise gates included.
\textbf{Result.} Pointwise gates are fast but unsound under an in-budget adversary; certified cost is backend-indexed (exact \(<1\,\mu\)s, Lipschitz \(\approx7.8\) ms, smoothing \(\approx10\)--\(46\) ms) (Table~\ref{tab:runtime}).

\begin{table}[h]
    \centering
    \small
    \caption{Validity regime: per-decision cost and soundness (Section~6.5). Pointwise gates are fast but unsound under an in-budget adversary; cost among the certified backends is backend-indexed (exact \(<\)1\,\(\mu\)s, Lipschitz \(\approx 7.8\) ms, smoothing \(\approx 10\)--\(46\) ms). The smoothing rows are on the MLP backend.}
    \label{tab:runtime}
    \resizebox{\columnwidth}{!}{%
        \begin{tabular}{llcccc}
            \toprule
            domain  & gate / backend                                      & branches \(\times M\) & mean ms    & \(\texttt{cert\_false\_allow}\) & \(\Rallow\) \\
            \midrule
            finance & none                                                & ---                   & 0.0002     & 0.667                           & 1.00        \\
            finance & rule                                                & ---                   & 0.004      & 0.500                           & 1.00        \\
            finance & learned                                             & ---                   & 0.080      & 0.500                           & 1.00        \\
            finance & exact (rung~1)                                      & \(9\times 1\)         & \(<0.001\) & 0.000                           & 1.00        \\
            finance & Lipschitz (\(\Allow_{\mathrm{Lip}}\))               & \(9\times 1\)         & 7.8        & 0.000                           & 0.26        \\
            finance & smoothing (\(\Allow_{\mathrm{RS}}\), \(M{=}2000\))  & \(9\times 2000\)      & 10.27      & 0.000                           & 0.50        \\
            finance & smoothing (\(\Allow_{\mathrm{RS}}\), \(M{=}10000\)) & \(9\times 10000\)     & 46         & 0.000                           & 0.30--0.33  \\
            SRE     & smoothing (\(\Allow_{\mathrm{RS}}\), \(M{=}2000\))  & \(9\times 2000\)      & 10.20      & 0.000                           & 0.22        \\
            \bottomrule
        \end{tabular}}
\end{table}

\medskip\noindent\textbf{Question.} Can gate--policy fidelity drift be monitored in operation?
\textbf{Setting.} A two-window delayed-audit detector over the real IEEE-CIS gate pool, swept over audit delay, window size, and alarm threshold (\(3\) seeds), with an anytime-valid upper bound.
\textbf{Result.} The monitor detects fidelity drift from audited allows alone, providing the operational complement to the one-shot certificates (Table~\ref{tab:monitor}).

\begin{table*}[t]
\centering
\small
\caption{Operational fidelity monitor (EXP-A4; Section~7). Two-window delayed-audit detector over the real IEEE-CIS gate pool (baseline audited $\texttt{cert\_false\_allow}\approx0.28$--$0.34$, the weak implicit signal); $\Delta_{\mathrm{audit}}\in\{1\mathrm{h},1\mathrm{d},7\mathrm{d}\}$, window $n\in\{500,2000,5000\}$, $\theta_{\mathrm{alarm}}\in\{0.02,0.05\}$, 3 seeds.
An anytime-valid upper confidence sequence ($\alpha$-spending Clopper--Pearson, peeking-safe) upgrades the audit to a finite-audit deployment guarantee on the same real stream: established at $N\approx500$--$1000$ with audited false-allow $\le0.021$ at $95\%$ uniformly over $N$, with an establish-then-halt fallback rule (control halt-rate $0$, and the bound holds even under the severe injected over-permissive gate; on a deterministic semi-synthetic rate jump $0.02\to0.12$ the rule halts $9{,}881$ audits past the jump). The cumulative sequence is a lifetime guarantee; the two-window detector above remains the complementary local change detector.}
\label{tab:monitor}
\resizebox{\textwidth}{!}{%
\begin{tabular}{lllll}
\toprule
injected regression & detection & control false alarms & exposure cut & latency \\
\midrule
label shift (strong) & $100\%$ at $15/18$ points & $0$ at those points & $6192\to69$--$489$ ($90$--$99\%$) & $\approx\Delta_{\mathrm{audit}}$ \\
over-permissive gate (subtle) & $33$--$67\%$ ($\theta=0.02$ only) & trade-off & partial & long \\
no-regression control & --- & rate $0.167$ (small-$n$/small-$\theta$ corner) & --- & --- \\
\bottomrule
\end{tabular}}
\end{table*}

\medskip\noindent\textbf{Question.} Is the observed \(C\) prevalence an artifact of one authored gap \(\delta\)?
\textbf{Setting.} The analytic taxonomy re-generated at \(\delta\in\{0.02,0.05,0.08,0.15,0.30\}\), \(\varepsilon=0.10\) fixed, natural sampling, on real IEEE-CIS and NAB (\(3\) seeds).
\textbf{Result.} \(\Pr[C]\) tracks \(\min(\delta,\varepsilon)\) and saturates above \(\varepsilon\), and the exact certificate is sound at every \(\delta\): the geometry predicts the prevalence (Table~\ref{tab:deltasweep}).

\begin{table}[h]
\centering
\small
\caption{$\delta$-sensitivity of $C$ prevalence (EXP-B1; Section~6.2). Analytic taxonomy at $\delta\in\{0.02,0.05,0.08,0.15,0.30\}$, $\varepsilon=0.10$ fixed, natural sampling, 3 seeds, on real IEEE-CIS ($n\approx295$k) and NAB ($n\approx29$k). $\Pr[C]$ tracks $\min(\delta,\varepsilon)$ and saturates above $\varepsilon$; the exact certificate has $\texttt{cert\_false\_allow}=0$ at every $\delta$. Levels are dataset-specific (boundary density); the $\min(\delta,\varepsilon)$ \emph{shape} is the claim.}
\label{tab:deltasweep}
\resizebox{\columnwidth}{!}{%
\begin{tabular}{lcccccc}
\toprule
dataset & $\delta{=}0.02$ & $0.05$ & $0.08$ & $0.15$ & $0.30$ & corr$(\Pr[C],\min(\delta,\varepsilon))$ \\
\midrule
IEEE-CIS & $0.011$ & $0.022$ & $0.033$ & $0.040$ & $0.040$ & $1.0$ \\
NAB      & $0.002$ & $0.012$ & $0.109$ & $0.191$ & $0.191$ & $0.956$ \\
\bottomrule
\end{tabular}}
\end{table}

\medskip\noindent\textbf{Question.} Was the numeric-block scaling outcome-conditioned?
\textbf{Setting.} \(f_{\mathrm{scale}}\) selected as the largest sound value on a disjoint SELECTION half and re-reported on the untouched EVAL half (grid \(\{2,3,4,6\}\), \(3\) seeds \(\times\,2\) tracks).
\textbf{Result.} Held-out selection reproduces the soundness headline: it is not outcome-conditioned hyperparameter selection (Table~\ref{tab:fscaleheldout}).

\begin{table}[h]
\centering
\small
\caption{Held-out $f_{\mathrm{scale}}$ selection (EXP-C4; selection hygiene for the numeric-block scaling of Tables~\ref{tab:nab} and~\ref{tab:dsweep}). $f_{\mathrm{scale}}$ chosen as the largest value with $\texttt{cert\_false\_allow}=0$ on a disjoint SELECTION half, re-reported on the untouched EVAL half (grid $\{2,3,4,6\}$, 3 seeds $\times$ 2 tracks): the soundness headline is not outcome-conditioned. The rule genuinely binds on the synthetic track (drops $f_{\mathrm{scale}}=2.0$ at selection $\texttt{cfa}=0.1$, seed 0): a small $f_{\mathrm{scale}}$ under-resolves the numeric block, so the gate misfits the boundary and its certified allows inherit the misfit. On NAB the same protocol selects $6.0$ (all seeds) with eval $\texttt{cert\_false\_allow}=0$ and $\Rallow=1.0$: soundness generalizes across the split with no cliff.}
\label{tab:fscaleheldout}
\resizebox{\columnwidth}{!}{%
\begin{tabular}{lllll}
\toprule
track & selected & eval $\texttt{cert\_false\_allow}$ & eval $\Rallow$ & clean acc \\
\midrule
synthetic & $6.0$ (rule binds) & $0$ (all seeds) & $0.967$ & $0.90$--$0.94$ \\
OPA finance & $6.0$ (no cliff) & $0$ (all seeds) & $0.781$ & $0.90$--$0.94$ \\
\bottomrule
\end{tabular}}
\end{table}

\medskip\noindent\textbf{Question.} Where does the regime of validity end in the continuous dimension?
\textbf{Setting.} A dimension sweep of the smoothed gate against the true oracle up to \(k=100\).
\textbf{Result.} Sound and non-vacuous through \(k=50\); at \(k=100\) the small smoothed gate stops tracking the oracle and the certificate faithfully certifies a weaker gate --- gate fidelity, not the certificate, is the bottleneck, and projection removes the edge (Table~\ref{tab:dimsweep}).

\begin{table}[h]
\centering
\small
\caption{Dimension sweep (Experiment D.2; anchors the \(k=100\) fidelity edge). Through \(k=50\) the certificate is sound and non-vacuous; at \(k=100\) the small smoothed gate stops tracking the oracle and the certificate faithfully certifies a weaker gate: gate fidelity, not the certificate, is the bottleneck. Policy-state projection removes this edge (Table~\ref{tab:projection}).}
\label{tab:dimsweep}
\resizebox{\columnwidth}{!}{%
\begin{tabular}{lccc}
\toprule
numeric dimension & clean acc & \(\texttt{cert\_false\_allow}\) & \(\Rallow\) \\
\midrule
\(k\le 50\)  & \(0.974\)--\(0.997\) & \(0\)      & \(0.38\)--\(0.60\) \\
\(k=100\)    & \(0.921\)            & \(0.23\)   & --- \\
\bottomrule
\end{tabular}}
\end{table}

\medskip\noindent\textbf{Question.} Does randomized-smoothing confidence survive a deployment horizon?
\textbf{Setting.} Certified \(\Rallow\) versus horizon \(T\) under a lifetime failure budget \(\alpha_{\mathrm{total}}=0.01\) (EXP-CX2).
\textbf{Result.} The frozen kill criterion fires: horizon-corrected RS goes vacuous, so long-horizon deployments fall back to the deterministic rungs (Table~\ref{tab:cx2horizon}).

\begin{table}[h]
\centering
\small
\caption{Deployment-horizon confidence for randomized smoothing (EXP-CX2; Section~5.1). Certified \(\Rallow\) on exact robust-safe records vs horizon \(T\) under a lifetime failure budget \(\alpha_{\mathrm{total}}=0.01\) (finance track shown; SRE/ops collapse to \(0\) sooner). The frozen kill criterion fires: horizon-corrected RS goes vacuous, per-record \(\alpha\) has a vacuous lifetime bound for \(T\ge10^3\), adaptive MC buys back utility at \(4\)--\(16\times\) sampling cost and stays below the deterministic certificate; the \(1\)-Lipschitz backend is horizon-invariant (no \(\alpha\), no MC; lifetime bound \(0\)).}
\label{tab:cx2horizon}
\resizebox{\columnwidth}{!}{%
\begin{tabular}{lcccc}
\toprule
scheme & \(T{=}10^3\) & \(10^4\) & \(10^5\) & \(10^6\) \\
\midrule
per-record \(\alpha{=}10^{-3}\) (lifetime bound vacuous) & \(0.12\) & \(0.12\) & \(0.12\) & \(0.12\) \\
Bonferroni \(\alpha_{\mathrm{total}}/T\) & \(0.035\) & \(0.020\) & \(0.010\) & \(0.005\) \\
\(\alpha\)-spending \(\propto1/t^2\) & \multicolumn{4}{c}{front-loaded: early \(0.21\), late \(\approx0\)} \\
adaptive MC (cap \(n_{\mathrm{mc}}{=}8000\)) & \multicolumn{4}{c}{\(0.22\)--\(0.29\) at \(4\)--\(16\times\) cost} \\
deterministic \method{}-Lip & \multicolumn{4}{c}{\(0.64/0.57/0.66\) (fin/SRE/ops), invariant in \(T\)} \\
\bottomrule
\end{tabular}}
\end{table}

\medskip\noindent\textbf{Question.} Does the shipped \method{}-Exact implementation match Definition~1 and Proposition~7?
\textbf{Setting.} A frozen-seed differential battery: \(200\) fragment policies \(\times\,1000\) returns (\(200{,}000\) total) against the independent OPA~1.17.1 engine evaluated at the exact per-constraint worst points, plus soundness-only ball sampling.
\textbf{Result.} The battery validates the implementation against the independent engine (Table~\ref{tab:cx3differential}).

\begin{table}[h]
\centering
\small
\caption{Differential validation of \method{}-Exact (EXP-CX3; the empirical leg of Definition~1 / Proposition~7). Frozen-seed battery: \(200\) fragment policies \(\times\,1000\) returns (\(200{,}000\) total; branches \(2\)--\(6\), \(m\le12\), \(k\le8\)) against the independent OPA~1.17.1 engine evaluated at the exact per-constraint worst points \(x+\varepsilon w/\|w\|_2\), plus soundness-only ball sampling and dense grids (\(k\le2\)); \(20\) out-of-fragment policies must be refused; the sub-\(\mu\)s figure of Table~S42 is the compiled scalar-threshold case.}
\label{tab:cx3differential}
\resizebox{\columnwidth}{!}{%
\begin{tabular}{lc}
\toprule
gate mismatches (both directions) & \(0\) / \(200{,}000\) \\
out-of-fragment refused (\texttt{unsupported}) & \(20/20\) \\
compiled-Rego vs Python point safety & \(0\) / \(500\) mismatches \\
median cost per decision (pure Python) & \(38.8\,\mu\mathrm{s}\) \\
log-log runtime slope in \(|\Ndisc|\,m\,k\) & \(0.64\) \\
\bottomrule
\end{tabular}}
\end{table}

\medskip\noindent\textbf{Question.} What autonomy does certification buy at operational operating points?
\textbf{Setting.} Implicit-\(\texttt{isFraud}\) triage on held-out mixed traffic (\(n=3000\), \(68\) fraud, \(3\) seeds), then the unconditional natural-traffic accounting (certified autonomy \(=\Pr[R]\cdot\Rallow\)).
\textbf{Result.} The certified strict-\(0\) frontier admits zero fraud on every seed; the operating points of the Takeaway (autonomy vs.\ in-budget fraud) come from these rows, and the exact rung's natural-traffic clearance is restated unconditionally (Tables~\ref{tab:triage}--\ref{tab:autonomy}).

\begin{table}[h]
    \centering
    \small
    \caption{Operational triage (Section~6.4): \(\Rallow\) as a certified-autonomy fraction on real IEEE-CIS (implicit \(\texttt{isFraud}\) policy, held-out mixed traffic \(n=3000\), \(68\) fraud, \(3\) seeds). ``In-budget fraud'' is the fraud fraction admitted in the autonomous tranche under the \(\Ball\) attack; the certified strict-\(0\) frontier admits \([0,0,0]\) per seed. Matching the certified volume without the robust-safe contract re-opens the exploit.}
    \label{tab:triage}
    \resizebox{0.9\columnwidth}{!}{%
        \begin{tabular}{lcc}
            \toprule
            gate                        & autonomous volume & in-budget fraud \\
            \midrule
            certified gate              & \(37\%\)          & \(0.044\)       \\
            point gate (volume-matched) & \(37\%\)          & \(0.532\)       \\
            point gate                  & \(80\%\)          & \(0.927\)       \\
            \bottomrule
        \end{tabular}}
\end{table}

\begin{table}[h]
\centering
\small
\caption{Natural-traffic autonomy accounting (Experiment M2; Section~6.4). \(\Rallow\) is conditional on the robust-safe set \(R\); this table restates it unconditionally on natural traffic: certified autonomy \(=\Pr[R]\cdot\Rallow\), human-review volume its complement, natural hazard \(\Pr[C]\) alongside (\(\delta=0.08\), \(\varepsilon=0.10\)). The exact rung-1 certificate on real data clears about half of all decisions autonomously; the learned rungs at the strict operating point (\(\varepsilon/\sigma=1\)) trade autonomy for a black-box-gate certificate, and that trade-off is now explicit per setting.}
\label{tab:autonomy}
\resizebox{\columnwidth}{!}{%
\begin{tabular}{llccc}
\toprule
setting & backend & natural \(\Pr[C]\) & certified autonomy & human review \\
\midrule
IEEE-CIS & exact (rung 1) & \(3.3\%\) & \(56.8\%\) & \(43.2\%\) \\
NAB telemetry & exact (rung 1) & \(10.9\%\) & \(46.1\%\) & \(53.9\%\) \\
PSD2/AML (REG) & RS (strict point) & \(6.5\)--\(9.8\%\) & \(2.0\)--\(9.2\%\) & \(90.8\)--\(98.0\%\) \\
OPA policy-as-code & Lipschitz & \(12.0\%\) & \(12.3\)--\(13.6\%\) & \(86.4\)--\(87.7\%\) \\
\bottomrule
\end{tabular}}
\end{table}

\medskip\noindent\textbf{Question.} What do the reported zeros mean statistically?
\textbf{Setting.} Every zero cell quoted in the main tables audited with explicit \(n/N\) and a Wilson-95\% upper bound, Clopper--Pearson cross-checked (\(91\) cells).
\textbf{Result.} All audited cells have \(k=0\) with honestly wide intervals where \(N\) is small, and each row carries an explicit denominator semantics (Table~\ref{tab:wilson}).

\begin{table}[h]
\centering
\small
\caption{Zero-cell audit (Experiment M1): explicit \(n/N\) and Wilson-95\% upper bound for the zeros quoted in the main tables; \(91\) cells audited, all \(k=0\), Clopper--Pearson cross-check agreeing throughout. Small-\(N\) cells are shown with their honestly wide intervals rather than a bare zero; each row carries an explicit denominator semantics (certified allows, Category-\(C\)/\(U\) records, exploit witnesses, episodes). Representative rows below; the full 91-cell audit ships with the code release.}
\label{tab:wilson}
\resizebox{\columnwidth}{!}{%
\begin{tabular}{lcc}
\toprule
cell (denominator semantics) & \(k/N\) & Wilson-95\% upper \\
\midrule
realistic schemas, \(\Uallow\) (finance; Category-\(U\) records) & \(0/17{,}200\) & \(2.2\times10^{-4}\) \\
e2e exploit, unsafe execution (joint cert / oracle; episodes) & \(0/300\) & \(0.0126\) \\
synthetic canonical, \(\Callow\), \(\Uallow\) (records) & \(0/40\) & \(0.0876\) \\
synthetic canonical, \(\texttt{cert\_false\_allow}\) (certified allows) & \(0/24\) & \(0.138\) \\
REG \texttt{psd2\_tra} natural \(\varepsilon{=}0.10\) (certified allows) & \(0/8\) & \(0.324\) \\
\bottomrule
\end{tabular}}
\end{table}


\end{document}